\documentclass[journal]{IEEEtran}
\usepackage{times}

\usepackage{multicol}
\usepackage[bookmarks=true]{hyperref}
\usepackage{amsmath}
\usepackage{framed}
\usepackage{amssymb}
\usepackage{bm}
\usepackage[version-1-compatibility]{siunitx}

\usepackage{booktabs}
\usepackage{amsthm}
\usepackage{graphicx}
\usepackage{tikz}

\tikzstyle axes label=[font=\footnotesize, text width=3cm, text centered]

\usepackage{color}
\usepackage{mathrsfs}

\graphicspath{{figures/}}

\theoremstyle{plain}
\newtheorem{theorem}{Theorem}
\newtheorem{lem}{Lemma}

\theoremstyle{definition}
\newtheorem{definition2}{Definition}
\newenvironment{definition}
  {
   \pushQED{\qed}\begin{definition2}}
  {\popQED\end{definition2}}

\theoremstyle{definition}

\newtheorem{remark2}{Implementation remark}
\newenvironment{remark}
  {%
   \pushQED{\qed}\begin{remark2}}
  {\popQED\end{remark2}}

\newtheorem{assum2}{Assumption}
\newenvironment{assum}
  {%
   \pushQED{\qed}\begin{assum2}}
  {\popQED\end{assum2}}

\usepackage[ruled, linesnumbered]{algorithm2e}
\let\oldnl\nl
\newcommand{\nonl}{\renewcommand{\nl}{\let\nl\oldnl}}
\SetKwIF{If}{ElseIf}{Else}{if}{ then}{elif}{else}{}%
\SetKwFor{For}{for}{ do}{}%
\SetKwFor{ForEach}{foreach}{ do}{}%
\SetKwInOut{Input}{Input}%
\SetKwInOut{Output}{Output}%
\AlgoDontDisplayBlockMarkers%
\SetAlgoNoEnd%
\SetAlgoNoLine%
\DontPrintSemicolon

\newcommand{\bftau}{\boldsymbol{\mathbf{\tau}}}

\renewcommand{\vec}[1]{\bm{\mathrm{#1}}}  

\newcommand{\cC}{\mathscr{C}}
\newcommand{\calX}{\mathcal{X}}
\newcommand{\calU}{\mathcal{U}}

\newcommand{\calK}{\mathcal{K}}

\newcommand{\rmend}{\mathrm{end}}
\newcommand{\bbI}{\mathbb{I}}

\newcommand{\calR}{\mathcal{R}}
\newcommand{\calQ}{\mathcal{Q}}
\newcommand{\calL}{\mathcal{L}}
\newcommand{\calP}{\mathcal{P}}

\newcommand{\calE}{\mathcal{E}}

\newcommand{\MVC}{\mathrm{MVC}}
\newcommand{\CLC}{\mathrm{CLC}}
\newcommand{\LC}{\mathrm{LC}}
\newcommand{\BLC}{\mathrm{BLC}}

\begin{document}

\title{\LARGE A New Approach to Time-Optimal Path Parameterization\\
  based on Reachability Analysis}

\author{\IEEEauthorblockN{Hung Pham, Quang-Cuong Pham}\\
  \IEEEauthorblockA{ATMRI, SC3DP, School of Mechanical and Aerospace Engineering\\
    Nanyang Technological University,  Singapore\\
    Email: pham0074@e.ntu.edu.sg, cuong.pham@normalesup.org} }


\maketitle

\begin{abstract}
  Time-Optimal Path Parameterization (TOPP) is a well-studied problem
  in robotics and has a wide range of applications. There are two main
  families of methods to address TOPP: Numerical Integration (NI) and
  Convex Optimization (CO). NI-based methods are fast but difficult to
  implement and suffer from robustness issues, while CO-based
  approaches are more robust but at the same time significantly
  slower. Here we propose a new approach to TOPP based on Reachability
  Analysis (RA). The key insight is to recursively compute reachable
  and controllable sets at discretized positions on the path by
  solving small Linear Programs (LPs). The resulting algorithm is
  faster than NI-based methods and as robust as CO-based ones (100\%
  success rate), as confirmed by extensive numerical
  evaluations. Moreover, the proposed approach offers unique
  additional benefits: Admissible Velocity Propagation and robustness
  to parametric uncertainty can be derived from it in a simple and
  natural way.
\end{abstract}

\IEEEpeerreviewmaketitle

\section{Introduction}
\label{sec:intro}

Time-Optimal Path Parameterization (TOPP) is the problem of finding
the fastest way to traverse a path in the configuration space of a
robot system while respecting the system
constraints~\cite{bobrow1985time}. This classical problem has a wide
range of applications in robotics. In many industrial processes
(cutting, welding, machining, 3D printing, etc.) or mobile robotics
applications (driverless cars, warehouse UGVs, aircraft taxiing, etc.),
the robot paths may be predefined, and optimal productivity implies
tracking those paths at the highest possible speed while respecting
the process and robot constraints. From a conceptual viewpoint, TOPP
has been used extensively as subroutine to kinodynamic motion planning
algorithms~\cite{SD91tra, pham2017admissible}. Because of its
practical and theoretical importance, TOPP has received considerable
attention since its inception in the 1980's,
see~\cite{pham2014general} for a recent review.

\subsection*{Existing approaches to TOPP}

There are two main families of methods to TOPP, based respectively on
Numerical Integration (NI) and Convex Optimization (CO). Each approach
has its strengths and weaknesses.

The NI-based approach was initiated by \cite{bobrow1985time}, and
further improved and extended by many researchers,
see~\cite{pham2014general} for a review. NI-based algorithms are based
on Pontryagin's Maximum Principle, which states that the time-optimal
path parameterization consists of alternatively maximally accelerating
and decelerating segments. The key advantage of this approach is that
the optimal controls can be explicitly computed (and not searched for
as in the CO approach) at each path position, resulting in extremely
fast implementations. However, this requires finding the switch points
between accelerating and decelerating segments, which constitutes a
major implementation difficulty as well as the main cause of
failure~\cite{pfeiffer1987concept,slotine1989improving,shiller1992computation,pham2014general}. Another
notable implementation difficulty is handling of velocity
bounds~\cite{Zla96icra}~\footnote{In a NI-based algorithm, to account
  for velocity bounds, one has to compute the direct Maximum Velocity
  Curve \(\mathrm{MVC}_{\mathrm{direct}}\), then find and resolve
  ``trap points''~\cite{Zla96icra}.  Implementing this procedure is
  tricky in practice because of accumulating numerical errors. This
  observation comes from our own experience with the TOPP
  library~\cite{pham2014general}.}. The formulation of the present
paper naturally removes those two difficulties.

The CO-based approach was initiated by~\cite{verscheure2008practical}
and further extended in~\cite{hauser2014}. This approach formulates
TOPP as a single large convex optimization program, whose optimization
variables are the accelerations and squared velocities at discretized
positions along the path. The main advantages of this approach are:
(i) it is simple to implement and robust, as one can use
off-the-shelf convex optimization packages; (ii) other convex
objectives than traversal time can be considered. On the downside, the
optimization program to solve is huge -- the number of variables and
constraint inequalities scale with the discretization step size --
resulting in implementations that are one order of magnitude slower
than NI-based methods~\cite{pham2014general}. This makes CO-based
methods inappropriate for online motion planning or as subroutine to
kinodynamic motion planners~\cite{pham2017admissible}.

\tikzstyle control set=[
  color={rgb:red, 85;green, 15; blue, 16}, >-<, very thick]
\tikzstyle greedy=[
  color={rgb:red, 31;green, 119; blue, 180}, ultra thick]
\tikzstyle greedy node=[
  fill={rgb:red, 31;green, 119; blue, 180}]
\begin{figure}[t]
  \hspace{-0.7cm}
  \begin{tikzpicture}
    \draw[->] (0, 0) -- + (6cm, 0);
    \draw[->] (0, 0) -- + (0cm, 3cm);
    \begin{scope}[x=20pt, y=20pt]
      \foreach \x in {0,1,...,8} {
        \draw[help lines, dashed] (\x, 0) -- (\x, 4);
        \node[axes label, anchor=north] at (\x, -0.15) {\x};
      }
      \draw [help lines] (8, 1.1) parabola (4, 4)
      (4, 4) -- (3, 2.5)
      (3, 2.5) parabola (0, 3.5);
      \draw [help lines] (8, 1.1) parabola (4, 0);

      \draw [thick]  (0, 1) -- +(-0.3, 0) ;
      \draw [thick]  (8, 1.1) -- +(+0.3, 0) ;

      \draw [control set ] (0, -0.1) -- (0, 3.6);
      \draw [control set ] (1, -0.1) -- (1, 3.15);
      \draw [control set ] (2, -0.1) -- (2, 2.85);
      \draw [control set ] (3, -0.1) -- (3, 2.75);
      \draw [control set ] (4, -0.1) -- (4, 4.1);
      \draw [control set ] (5, 0.3) -- (5, 2.95);
      \draw [control set ] (6, 0.65) -- (6, 2.05);
      \draw [control set ] (7, 0.88) -- (7, 1.5);
      \draw [control set ] (8, 0.9) -- (8, 1.3);

      \node [axes label] at   (8.5, 3.8) {(bw)};
      \draw [control set, ->] (7.9, 3.8) -- (7.1, 3.8);

      \node [axes label] at  (-.5, 3.8) {(fw)};
      \draw [greedy , <-] (0.9, 3.8) -- (0.1, 3.8);

      \draw[greedy node] (0, 1   ) node (l0) {} circle (2pt);
      \draw[greedy node] (1, 2.9) node (l1) {} circle (2pt);
      \draw[greedy node] (2, 2.6 ) node (l2) {} circle (2pt);
      \draw[greedy node] (3, 2.5) node (l3) {} circle (2pt);
      \draw[greedy node] (4, 2.4 ) node (l4) {} circle (2pt);
      \draw[greedy node] (5, 2.7) node (l5) {} circle (2pt);
      \draw[greedy node] (6, 1.8) node (l6) {} circle (2pt);
      \draw[greedy node] (7, 1.25) node (l7) {} circle (2pt);
      \draw[greedy node] (8, 1.10) node (l8) {} circle (2pt);
      \draw [greedy, densely dashed] (l0) --
                     (l1) --
                     (l2) --
                     (l3) --
                     (l4) --
                     (l5) --
                     (l6) --
                     (l7) --
                     (l8) ;

      \node [axes label, anchor=north] at (4, -0.7) {$s$};
      \node [axes label, anchor=east, text width=0] at (-0.3, 2)
      {$\dot s$};
      \node [axes label, anchor=east, text width=0] at (-0.7, 1)
      {$\dot s_0^2$};
      \node [axes label, anchor=west, text width=0] at (8.3, 1.1) {$\dot
        s_N^2$};

    \end{scope}
  \end{tikzpicture}
  \caption{\label{fig:topp-ra} Time-Optimal Path Parameterization by
    Reachability Analysis (TOPP-RA) computes the optimal
    parameterization in two passes. In the first pass (backward),
    starting from the last grid point $N$, the algorithm computes
    controllable sets (red intervals) recursively. In the second pass
    (forward), starting now from grid point $0$, the algorithm
    greedily selects the highest controls such that resulting
    velocities remain inside the respective controllable sets.}
\end{figure}
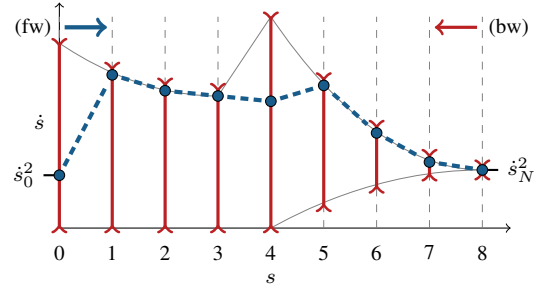

\subsection*{Proposed new approach based on Reachability Analysis}

In this paper, we propose a new approach to TOPP based on Reachability
Analysis (RA), a standard notion from control theory. The key insight
is: given an interval of squared velocities $\bbI_s$ at some position $s$ on
the path, the \emph{reachable} set $\bbI_{s+\Delta}$ (the set of all
squared velocities at the next path position that can be reached from $\bbI_s$
following admissible controls) and the \emph{controllable} set
$\bbI_{s-\Delta}$ (the set of all squared velocities at the previous path
position such that there exists an admissible control leading to a
velocity in $\bbI_s$) can be computed quickly and robustly by solving
a few \emph{small} Linear Programs (LPs). By recursively computing
controllable sets at discretized positions on the path,
one can then extract the time-optimal parameterization in time
$O(mN)$, where $m$ is the number of constraint inequalities and $N$
the discretization grid size, see Fig.~\ref{fig:topp-ra} for an
illustration.

As compared to NI-based methods, the proposed approach has therefore a
better time complexity (actual computation time is similar for problem
instances with few constraints, and becomes significantly faster for
instances with $>22$ constraints). More importantly, the proposed
method is much easier to implement and has a success rate of $100\%$,
while state-of-the-art NI-based implementations
(e.g.~\cite{pham2014general}) comprise thousands of lines of code and
still report failures on hard problem instances. As compared to
CO-based methods, the proposed approach enjoys the same level of
robustness and of ease-of-implementation while being significantly
faster.

Besides the gains in implementation robustness and performance,
viewing the classical TOPP problem from the proposed new perspective
yields the following additional benefits:
\begin{itemize}
\item constraints for redundantly-actuated systems are handled
  natively: there is no need to project the constraints to the plane
  (path acceleration $\times$ control) at each path position, as done
  in~\cite{hauser2014,pham2015time};
\item Admissible Velocity Propagation~\cite{pham2017admissible}, a
  recent concept for kinodynamic motion planning (see
  Section~\ref{sec:avp} for a brief summary), can be derived ``for
  free'';
\item robustness to parametric uncertainty, e.g.\ uncertain
  coefficients of friction or uncertain inertia matrices, can be
  obtained in a natural way.
\end{itemize}
More details regarding the benefits as well as definitions of relevant
concepts will be given in Section~\ref{sec:optimal-path-policy}.

\subsection*{Organization of the paper}

The rest of the paper is organized as follows.
Section~\ref{sec:nomi-path-param} formulates the TOPP problem in a
general setting. Section~\ref{sec:reach-analys-} applies Reachability
Analysis to the path-projected
dynamics. Section~\ref{sec:time-optimal-profile} presents the core
algorithm to compute the time-optimal path
parameterization. Section~\ref{sec:experiments} reports extensive
experimental results to demonstrate the gains in robustness and
performance permitted by the new
approach. Section~\ref{sec:optimal-path-policy} discusses the
additional benefits mentioned previously: Admissible Velocity
Propagation and robustness to parametric uncertainty. Finally,
Section~\ref{sec:discussion} offers some concluding remarks and
directions for future research.

\section{Problem formulation}
\label{sec:nomi-path-param}

\subsection{Generalized constraints}
\label{sec:path-param-probl}

Consider a $n$-dof robot system, whose configuration is denoted by a
$n$ dimensional vector $\mathbf{q} \in \mathbb{R}^{n}$.  A
\emph{geometric path} $\mathcal{P}$ in the configuration space is
represented as a function $\mathbf{q}(s)_{s\in[0, s_\rmend]}$. We
assume that $\mathbf{q}(s)$ is piece-wise
$\mathcal{C}^{2}$-continuous. A \emph{time parameterization} is a
piece-wise $\mathcal{C}^2$, increasing scalar function
$s:[0, T] \rightarrow [0, s_\rmend]$, from which a \emph{trajectory}
is recovered as $\vec q (s(t))_{t\in[0, T]}$.

In this paper, we consider \emph{generalized second-order
  constraints} of the following form~\cite{hauser2014,pham2015time}
\begin{equation}
  \label{eq:general-form}
  \mathbf{A(q)} \ddot{\mathbf{q}} +
  \dot{\mathbf{q}}^{\top}\mathbf{B(q)}\dot{\mathbf{q}} +
  \mathbf{f(q)} \in \cC(\vec q),\ \textrm{where}
\end{equation}
\begin{itemize}
\item $\mathbf{A, B, f}$ are continuous mappings from $\mathbb{R}^n$
  to $\mathbb{R}^{m\times n},\mathbb{R}^{n\times m\times n}$ and
  $\mathbb{R}^{m}$ respectively;
\item $\cC(\vec q)$ is a convex polytope in $\mathbb{R}^{m}$.
\end{itemize}

\begin{remark}
The above form is the most general in the TOPP literature to date,
and can account for many types of kinodynamic constraints, including
velocity and acceleration bounds, joint torque bounds for fully- or
redundantly-actuated robots~\cite{pham2015time}, contact stability
under Coulomb friction
model~\cite{hauser2014,caron2015leveraging,Caron2017}, etc.

Consider for instance the torque bounds on a fully-actuated
manipulator
\begin{align}
  \label{eq:manip}
  &\vec M(\vec q)\ddot{\vec q}+
  \dot{\vec q}^\top\vec C(\vec q)\dot{\vec q}+
  \vec g(\vec q)=\bftau, \\
  &\tau^{\min}_i \leq \tau_i(t) \leq \tau^{\max}_i, \ \forall
  i\in[1,\dots,n],\ t\in [0,T]
\end{align}
This can be rewritten in the form of~\eqref{eq:general-form} with $\vec A:= \vec
M$, $\vec B:=\vec C$, $\vec f:=\vec g$ and
\[
\cC(\vec q):= [\tau^{\min}_1,\tau^{\max}_1] \times
\dots \times [\tau^{\min}_n,\tau^{\max}_n],
\]
which is clearly convex.

For redundantly-actuated manipulators, it was shown that the TOPP
problem can also be formulated in the form
of~\eqref{eq:general-form}~\cite{pham2015time} with
\[
\cC(\vec q):= \vec S^\top \left([\tau^{\min}_1,\tau^{\max}_1] \times
\dots \times [\tau^{\min}_n,\tau^{\max}_n]\right),
\]
where $\vec S$ is a linear transformation~\cite{pham2015time}, which
implies that the so-defined $\cC(\vec q)$ is a convex polytope.

In legged robots, the TOPP problem under contact-stability constraints
where the friction cones are linearized was shown to be reducible to
the form of~\eqref{eq:general-form} with $\cC(\vec q)$ being also a
\emph{convex
  polytope}~\cite{hauser2014,pham2015time,caron2015leveraging}.

If the friction cones are not linearized, then $\cC(\vec q)$ is still
convex, but not polytopic. The developments in the present paper that
concern reachable and controllable sets
(Section~\ref{sec:reach-analys-}) are still valid in the convex,
non-polytopic case. The developments on time-optimality
(Section~\ref{sec:time-optimal-profile}) is however only applicable to
the polytopic case.
\end{remark}



Finally, we also consider first-order constraints of the form
\begin{equation*}
  \mathbf{A}^v(\vec q) \dot{\mathbf{q}} +
  \mathbf{f}^v(\vec q) \in \cC^v(\vec q),
\end{equation*}
where the coefficients are matrices of appropriate sizes and
$\cC^v(\vec q)$ is a convex set. Direct velocity bounds and momentum
bounds are examples of first-order constraints.

\subsection{Projecting the constraints on the path}
\label{sec:proj-constr-geom}

Differentiating successively $\vec q(s)$, one has
\begin{equation}
  \label{eq:differentiate-q}
  \dot {\mathbf{q}} = \mathbf{q}'\dot s, \quad
  \ddot {\mathbf{q}} = \mathbf{q}''\dot s^2 + \mathbf{q}'\ddot s,
\end{equation}
where $\Box'$ denotes differentiation with respect to the path
parameter $s$. From now on, we shall refer to $s, \dot s, \ddot s$ as
the position, velocity and acceleration respectively.

Substituting Eq.~\eqref{eq:differentiate-q} to Eq.~\eqref{eq:general-form}, one transforms
second-order constraints on the system dynamics into constraints on
$s, \dot s, \ddot s$ as follows
\begin{equation}
  \label{eq:lp-in-path}
    \mathbf{a}(s) \ddot s + \mathbf{b}(s) \dot s^2+ \mathbf{c}(s) \in
    \cC(s),\ \textrm{where}
\end{equation}
\begin{equation*}
  \begin{aligned}
    \mathbf{a}(s)&:=\vec A (\vec q (s)) \vec{q}'(s),\\
    \mathbf{b}(s)&:=\vec A (\vec q (s))\vec{q}''(s) +
    \vec{q}'(s)^\top \vec B (\vec q (s))\vec{q}'(s),\\
    \mathbf{c}(s)&:= \vec f (\vec q (s)), \\
    \cC(s) &:= \cC(\vec q(s)).
  \end{aligned}
\end{equation*}


Similarly, first-order constraints are transformed into
\begin{equation}
  \label{eq:lp-in-path-direct}
  \mathbf{a}^v(s) \dot s + \mathbf{b}^v(s) \in \cC^v(s),\ \textrm{where}
\end{equation}
\begin{equation*}
  \begin{aligned}
    \mathbf{a}^v(s)&:=\vec A^v (\vec q (s)) \vec{q}'(s),\\
    \mathbf{b}^v(s)&:= \vec f^v (\vec q (s)),\\
    \cC^v(s) &:= \cC^v(\vec q(s)).
  \end{aligned}
\end{equation*}


\subsection{Path discretization}
\label{sec:discr-path-dynam}

As in the CO-based approach, we divide the interval
$[0, s_{\rmend}]$ into $N$ segments and $N+1$ grid points
\begin{equation*}
  0=:s_0, s_1 \dots s_{N-1}, s_N:= s_{\rmend}.
\end{equation*}

Denote by $u_i$ the constant path acceleration over the interval
$[s_i, s_{i+1}]$ and by $x_i$ the squared velocity $\dot s_i^2$ at
$s_i$. By simple algebraic manipulations, one can show that the
following relation holds
\begin{equation}
  \label{eq:lin-rls}
  x_{i+1} = x_i + 2 \Delta_i u_i, \quad i = 0\dots N-1,
\end{equation}
where $\Delta_i:= s_{i+1}-s_{i}$. In the sequel we refer to $s_i$ as
the $i$-stage, $u_i$ and $x_i$ as respectively the control and state
at the $i$-stage. Any sequence
$x_0, u_0, \dots, x_{N-1}, u_{N-1}, x_N$ that satisfies the linear
relation~\eqref{eq:lin-rls} is referred to as a path parameterization.

A parameterization is \emph{admissible} if it satisfies the
constraints at every points in $[0, s_{\rmend}]$. One possible way to
bring this requirement into the discrete setting is through a
\emph{collocation} discretization scheme: for each position $s_i$, one
evaluates the continuous constraints and requires the control and
state $u_i, x_i$ to verify
\begin{equation}
  \label{eq:gen-form}
    \mathbf{a}_i u_i + \mathbf{b}_i x_i+ \mathbf{c}_i \in \cC_i,
\end{equation}
where
$\vec a_i:=\vec a(s_i), \vec b_i:=\vec b(s_i), \vec c_i:=\vec c(s_i),
\cC_i:=\cC(s_i)$.

Since the constraints are enforced only at a finite number of points,
the actual continuous constraints might not be respected everywhere
along $[0, s_\rmend]$\,\footnote{This limitation is however not
  specific to the proposed approach as both the NI and CO approaches
  require discretization at some stages of the algorithm.}. Therefore,
it is important to bound the constraint satisfaction error. We show in
Appendix~\ref{sec:discr-scheme-error} that the collocation scheme has
an error of order $O(\Delta_i)$. Appendix~\ref{sec:discr-scheme-error}
also presents a first-order interpolation discretization scheme, which
has an error of order $O(\Delta_i^2)$ but which involves more
variables and inequality constraints than the collocation scheme.

\section{Reachability Analysis of the path-projected dynamics}
\label{sec:reach-analys-}

The key to our analysis is that the ``path-projected
dynamics''~\eqref{eq:lin-rls},~\eqref{eq:gen-form} is a
\emph{discrete-time linear system} with \emph{linear control-state
  inequality constraints}. This observation immediately allows us to
take advantage of the set-membership control problems studied in the
Model Predictive Control (MPC)
literature~\cite{kerrigan2001robust,Bertsekas1971,rakovic2006reachability}.


\subsection{Admissible states and controls}
\label{sec:useful-definitions}

We first need some definitions. Denote the $i$-stage set of
\emph{admissible} control-state pairs by
\begin{equation*}
  \Omega_i:= \{(u, x) \mid \mathbf{a}_i u + \mathbf{b}_i x+ \mathbf{c}_i \in \cC_i\}.
\end{equation*}
One can see $\Omega_i$ as the projection of $\cC_i$ on the $(\ddot
s,\dot s^2)$ plane~\cite{hauser2014}. Since $\cC_i$ is a polytope,
$\Omega_i$ is a \emph{polygon}. Algorithmically, the projection can be
obtained by e.g.\ the recursive expansion
algorithm~\cite{bretl2008tro}.

Next, the $i$-stage set of \emph{admissible states} is the projection of
$\Omega_{i}$ on the second axis
\begin{equation*}
  \calX_i:= \{x \mid \exists u: (u,x) \in \Omega_i \}.
\end{equation*}
The $i$-stage set of \emph{admissible controls} given a state $x$ is
\begin{equation*}
  \calU_i(x):= \{u \mid (u, x) \in \Omega_i \}.
\end{equation*}

Note that, since $\Omega_i$ is convex, both $\calX_i$ and $\calU_i(x)$
are \emph{intervals}.

Classic terminologies in the TOPP literature (e.g. Maximum Velocity
Curve, $\alpha$ and $\beta$ acceleration fields, etc.) can be
conveniently expressed using these definitions. See the first part of
Appendix~\ref{sec:relat-with-numer} for more details.

\begin{remark}
  \label{rem:redundant}
 For redundantly-actuated manipulators and contact-stability of legged
 robots, both NI-based and CO-based methods must compute $\Omega_i$ at
 each discretized position $i$ along the path, which is costly. Our
 proposed approach avoids performing this 2D projection: instead, it
 will only require a few 1D projections per discretization
 step. Furthermore, each of these 1D projections amounts to a pair of
 LPs and can therefore be performed extremely quickly.
\end{remark}

\subsection{Reachable sets}
\label{sec:reachable-sets}



The key notion in Reachability Analysis is that of $i$-stage reachable
set.

\begin{definition}[$i$-stage reachable set]
  Consider a set of starting states $\bbI_0$. The \emph{$i$-stage
    reachable set} $\mathcal{L}_i(\bbI_0)$ is the set of states
  $x\in\calX_i$ such that there exist a state $x_0\in \bbI_0$ and a
  sequence of admissible controls $u_0,\dots, u_{i-1}$ that steers the
  system from $x_0$ to $x$.
\end{definition}

To compute the $i$-stage reachable set, one needs the following
intermediate representation.
\begin{definition}[Reach set]
  Consider a set of states $\bbI$. The \emph{reach set}
  $\calR_i(\bbI)$ is the set of states $x \in \calX_{i+1}$ such that
  there exist a state $\tilde{x}\in\bbI$ and an admissible control
  $u\in\calU_i(\tilde{x})$ that steers the system from $\tilde{x}$ to
  $x$, i.e.
  \[
  x = \tilde{x} + 2\Delta_i u. \qedhere
  \]
\end{definition}

\begin{remark}
  \label{rem:convex}
  Let us note
  $\Omega_i(\bbI) := \{(u,\tilde{x})\in\Omega_i \mid
  \tilde{x}\in\bbI\}$.
  If $\bbI$ is convex, then $\Omega_i(\bbI)$ is convex as the
  intersection of two convex sets.  Next, $\calR_i(\bbI)$ can be seen
  as the intersection of the projection of $\Omega_i(\bbI)$ onto a line
  and the interval $\calX_{i+1}$. Thus,
  $\calR_i(\bbI)$ is an interval, hence defined by its lower and upper
  bounds $(x^-,x^+)$, which can be computed as follows
  \[
  x^- := \min_{(u,\tilde{x})\in\Omega_i(\bbI),\; x^- \in \calX_{i+1}} \tilde{x} + 2\Delta_i u,
  \]
  \[
  x^+ := \max_{(u,\tilde{x})\in\Omega_i(\bbI),\; x^+ \in \calX_{i+1}} \tilde{x} + 2\Delta_i u.
  \]
  Since $\Omega_i(\bbI)$ is a polygon, the above equations constitute
  two LPs. Note finally that there is no need to compute
  explicitly $\Omega_i(\bbI)$, since one can write directly
  \[
    \begin{aligned}
      &x^+ := \max_{(u,\tilde{x})\in\mathbb{R}^2} \tilde{x} + 2\Delta_i u, \\
      \textrm{subject to:~~} &\mathbf{a}_i u + \mathbf{b}_i \tilde{x}+ \mathbf{c}_i
      \in \cC_i, \tilde{x}\in\bbI  \textrm{~and~} x^{+} \in \calX_{i+1},
    \end{aligned}
  \]
  and similarly for $x^-$.
\end{remark}

The $i$-stage reachable set can be recursively computed by
\begin{equation}
  \label{eq:rch-com}
  \begin{aligned}
    \calL_0(\bbI_0) &= \bbI_0 \cap \calX_0,\\
    \calL_i(\bbI_0) &= \calR_{i-1}(\calL_{i-1}(\bbI_0)).
  \end{aligned}
\end{equation}

\begin{remark}
  \label{rem:reachable}
  If $\bbI_0$ is an interval, then by recursion and by application of
  Implementation remark~\ref{rem:convex}, all the $\calL_i$ are
  intervals. Each step of the recursion requires solving two LPs for
  computing $\calR_{i-1}(\calL_{i-1}(\bbI_0))$. Therefore, $\calL_i$
  can be computed by solving $2i + 2$ LPs.
\end{remark}

The $i$-stage reachable set may be empty, which implies that the
system can not evolve without violating constraints: the path is not
time-parameterizable. One can also note that
\begin{equation*}
  \calL_i(\bbI_0) = \emptyset \implies   \forall j \geq i,\ \calL_j (\bbI_0) =
  \emptyset.
\end{equation*}

\subsection{Controllable sets}
\label{sec:controllable-sets}

Controllability is the dual notion of reachability, as made clear by
the following definitions.

\begin{definition}[$i$-stage controllable set]
  Consider a set of desired ending states $\bbI_N$. The
  \emph{$i$-stage controllable set} $\mathcal{K}_i(\bbI_N)$ is the set
  of states $x\in\calX_i$ such that there exist a state
  $x_N\in \bbI_N$ and a sequence of admissible controls
  $u_i,\dots, u_{N-1}$ that steers the system from $x$ to $x_N$.
\end{definition}

The dual notion of ``reach set'' is that of ``one-step'' set.
\begin{definition}[One-step set]
  Consider a set of states $\bbI$. The \emph{one-step set}
  $\calQ_i(\bbI)$ is the set of states $x\in\calX_{i}$ such that there exist a
  state $\tilde{x}\in\bbI$ and an admissible control $u\in\calU_i(x)$ that
  steers the system from $x$ to $\tilde{x}$, i.e.
  \[
  \tilde{x} = x + 2\Delta_i u. \qedhere
  \]
\end{definition}

The $i$-stage controllable set can now be computed recursively by
\begin{equation}
  \label{eq:ctrl-comp}
  \begin{aligned}
    \calK_N(\bbI_N) &= \bbI_N \cap \calX_N,\\
    \calK_i(\bbI_N) &= \calQ_i(\calK_{i+1}(\bbI_N)).
  \end{aligned}
\end{equation}

\begin{remark}
\label{rem:controllable}
Similar to Implementation remark~\ref{rem:reachable},
every one-step set $\calQ_i(\bbI)$ is an interval, whose lower and
upper bounds $(x^-,x^+)$ are given by the following two LPs
\[
  \begin{aligned}
&x^+ := \max_{(u,x)\in\mathbb R^2} x, \\
\textrm{subject to:~} &\mathbf{a}_i u + \mathbf{b}_i x+ \mathbf{c}_i
\in \cC_i \ \textrm{~and~}  x + 2\Delta_i u \in\bbI,
  \end{aligned}
\]
and similarly for $x^-$. Thus, computing the $i$-stage controllable
set will require solving $2 (N - i) + 2$ LPs.
\end{remark}

The $i$-stage controllable set may be empty, in that case, the path is
not time-parameterizable. One also has
\begin{equation*}
  \calK_i(\bbI_N) = \emptyset \implies \forall j \leq
  i,\ \calK_j(\bbI_N) = \emptyset.
\end{equation*}

\section{TOPP by Reachability Analysis}
\label{sec:time-optimal-profile}

\subsection{Algorithm}

Armed with the notions of reachable and controllable sets, we can now
proceed to solving the TOPP problem. The Reachability-Analysis-based
TOPP algorithm (TOPP-RA) is given in Algorithm~\ref{algo:topp} below
and illustrated in Fig.~\ref{fig:topp-ra}.

\begin{algorithm}[h]
\label{algo:topp}
  \caption{TOPP-RA}
  \Input{Path $\calP$, starting and ending velocities $\dot s_0, \dot s_N$} %
  \Output{Parameterization $x_0^*, u_0^*, \dots, u_{N-1}^*, x_N^*$} %
  \tcc{Backward pass: compute the controllable sets}
  $\calK_N := \{\dot s_N^{2}\}$ \;
  \For{$i \in [N-1 \dots 0]$}{
    $\calK_i := \calQ_i(\calK_{i+1})$
  }
  \If{$\calK_0 = \emptyset$ \rm{or} $\dot s_0^{2} \notin \calK_0$}{
    \Return \texttt{Infeasible}
  }
  \tcc{Forward pass: select controls greedily}
  $x^*_0 := \dot s_0^{2}$ \;
  \For{$i \in [0 \dots N-1]$}{
    $u_{i}^* := \max u$,
    subject to: $x^*_i + 2 \Delta_iu \in
    \calK_{i+1}$ and $(u, x^*_i) \in \Omega_i$  \;
    $x_{i+1}^* := x_{i}^* + 2 \Delta_{i} u_{i}^*$
  }
\end{algorithm}

The algorithm proceeds in two passes. The first pass goes backward: it
recursively computes the controllable sets $\calK_i(\{\dot s_N^2\})$
given the desired ending velocity $\dot s_N$, as described in
Section~\ref{sec:controllable-sets}. If any of the controllable sets
is empty or if the starting state $\dot s_0^{2}$ is not contained in
the 0-stage controllable set, then the algorithm reports failure. 

Otherwise, the algorithm proceeds to a second, forward, pass. Here,
the optimal states and controls are constructed \emph{greedily}: at
each stage $i$, the highest admissible control $u$ such that the
resulting next state belongs to the $(i+1)$-stage controllable set is
selected.

Note that one can construct a ``dual version'' of TOPP-RA as
follows: (i) in a forward pass, recursively compute the $i$-stage
reachable sets, $i\in[0,\dots,N]$; (ii) in a backward pass, greedily
select, at stage $i$, the lowest control such that the 
previous state belongs to the $(i-1)$-stage reachable set. 

In the following sections, we show the correctness and optimality of
the algorithm and give a more detailed complexity analysis.

\subsection{Correctness of TOPP-RA}
\label{sec:feasibility-topp-ra}

We show that TOPP-RA is correct in the sense of the following
theorem. 

\begin{theorem}
  Consider a discretized TOPP instance. TOPP-RA returns an admissible
  parameterization solving that instance whenever one exists, and
  reports \rm{\texttt{Infeasible}} otherwise.
\end{theorem}

\begin{proof}
  (1) We first show that, if TOPP-RA reports \texttt{Infeasible}, then
  the instance is indeed not parameterizable. By contradiction, assume
  that there exists an admissible parameterization $\dot s_0^2 = x_0,
  u_0, \dots, u_{N-1}, x_N = \dot s_N^2$. We now show by backward
  induction on $i$ that $\calK_i$ contains at least $x_i$.

  \textrm{Initialization}: $\calK_N$ contains $x_N$ by construction.

  \textrm{Induction}: Assume that $\calK_i$ contains $x_i$. Since the
  parameterization is admissible, one has $x_i=x_{i-1}+2\Delta_iu_{i-1}$
  and $(u_{i-1},x_{i-1})\in\Omega_{i-1}$. By definition of the
  controllable sets, $x_{i-1}\in\calK_{i-1}$.

  We have thus shown that none of the $\calK_i$ is empty and that
  $\calK_0$ contains at least $x_0 = \dot s_0^2$, which implies that
  TOPP-RA cannot report \texttt{Infeasible}.

  (2) Assume now that TOPP-RA returns a sequence
  $(x_0^*, u_0^*, \dots, u_{N-1}^*, x_N^*)$. One can easily show by
  forward induction on $i$ that the sequence indeed constitutes an
  admissible parameterization that solves the instance.
\end{proof}

\subsection{Asymptotic optimality of TOPP-RA}
\label{sec:proof-correctness}

We show the following result: as the discretization step size goes to
zero, the cost, i.e.~traversal time, of the parameterization returned
by TOPP-RA converges to the optimal value.

Unsurprisingly, the main difficulty with proving asymptotic optimality
comes from the existence of zero-inertia
points~\cite{pfeiffer1987concept,shiller1992computation,pham2014general}. Note
however this difficulty does not affect the robustness or
the correctness of the algorithm.

To avoid too many technicalities, we make the following assumption.



\begin{assum}[and definition]
\label{assum:1}
There exist piece-wise $\mathcal{C}^{1}$-continuous functions
$\tilde{\vec a}(s)_{s\in [0, 1]}, \tilde{\vec b}(s)_{s\in [0, 1]},
\tilde{\vec c}(s)_{s\in [0, 1]}$ such that for all
$i\in \{0,\dots,N\}$, the set of admissible control-state pairs is
given by
\begin{equation*}
  \Omega_{i} = \{(u,x)\mid u\tilde{\vec a}(s_{i}) + x\tilde{\vec
    b}(s_{i}) + \tilde{\vec c}(s_{i}) \leq 0\}.
\end{equation*}
Augment $\tilde{\vec a},\tilde{\vec b},\tilde{\vec c}$ into
$\bar{\vec a},\bar{\vec b},\bar{\vec c}$ by adding two inequalities
that express the condition $x+2\Delta_{i}u\in{\cal K}_{i+1}$. The set of admissible
\emph{and controllable} control-state pairs is given by
\begin{equation*}
  \Omega_{i}\cap (\mathbb{R}\times \mathcal{K}_{i} )= \{(u,x)\mid u\bar{\vec a}(s_{i}) + x\bar{\vec
    b}(s_{i}) + \bar{\vec c}(s_{i}) \leq 0\}.\qedhere
\end{equation*}
\end{assum}

The above assumption is easily verified in the canonical case of a
fully-actuated manipulator subject to torque bounds tracking a smooth
path. It allows us to next easily define zero-inertia points.

\begin{definition}[Zero-inertia points]
  A point $s^\bullet$ constitutes a zero-inertia point if there is a
  constraint $k$ such that  \mbox{$\bar{\vec a}(s^\bullet)[k]=0$}.
\end{definition}



We have the following theorem, whose proof is given in
Appendix~\ref{sec:proof-no} (to simplify the notations, we consider
uniform step sizes $\Delta_0 =\dots=\Delta_{N-1} = \Delta$).

\begin{theorem}
  \label{theo:1}
  Consider a TOPP instance \emph{without zero-inertia points}. There
  exists a $\Delta_\mathrm{thr}$ such that if $\Delta <
  \Delta_\mathrm{thr}$, then the parameterization returned by TOPP-RA
  is optimal.
\end{theorem}

The key hypothesis of this theorem is that there is no zero-inertia
points.  In practice, however, zero-inertia points are unavoidable and
in fact constitute the most common type of switch
points~\cite{pham2014general}.  The next theorem, whose proof is given
in Appendix~\ref{sec:proof-with}, establishes that the sub-optimality
gap converges to zero with step size.

\begin{theorem}
  \label{theo:2}
  Consider a TOPP instance with a zero-inertia point at
  $s^\bullet$. Denote by $J^*$ the cost of the parameterization
  returned by TOPP-RA at step size $\Delta$:
  $\sum_{i=0}^{N+1}\frac{\Delta}{\sqrt{x_{i}^{*}}}$ and by $J^\dagger$ the
  minimum cost at the same step size. Then one has
  \[
  J^* - J^\dagger = O(\Delta).
  \]
\end{theorem}



This theorem implies that, by reducing the step size, the cost of the
parameterization returned by TOPP-RA can be made arbitrarily close to
the minimum cost. This remains true when there are a finite number of
zero-inertia points. The case of zero-inertia
\emph{arcs}~\cite{shiller1992computation} might be more problematic,
but it is always possible to avoid such arcs during the planning
stage.

\subsection{Complexity analysis}
\label{sec:runn-time-analys}

We now perform a complexity analysis of TOPP-RA and compare it with
the Numerical Integration and the Convex Optimization approaches. For
simplicity, we shall restrict the discussion to the non-redundantly
actuated case (the redundantly-actuated case actually brings an
additional advantage to TOPP-RA, see Implementation
remark~\ref{rem:convex}).

Assume that there are $m$ constraint inequalities and that the path
discretization grid size is $N$.  As a large part of the computation
time is devoted to solving LPs, we need a good estimate of the
practical complexity of this operation. Consider a LP with $\nu$
optimization variables and $m$ inequality constraints. Different LP
methods (ellipsoidal, simplex, active sets, etc.) have different
complexities. For the purpose of this section, we consider the best
\emph{practical} complexity, which is realized by the simplex method,
in $O(\nu^2 m)$~\cite{Boyd:2004:CO:993483}.

\begin{itemize}
\item \emph{TOPP-RA:} The LPs considered here have 2 variables and
  $m+2$ inequalities. Since one needs to solve $3N$ such LPs, the
  complexity of TOPP-RA is $O(mN)$.

\item \emph{Numerical integration approach:} The dominant component of
  this approach, in terms of time complexity, is the computation of
  the Maximum Velocity Curve (MVC). In most TOPP-NI implementations to
  date, the MVC is computed, at each discretized path position, by
  solving $O(m^2)$ second-order
  polynomials~\cite{bobrow1985time,pfeiffer1987concept,
    slotine1989improving,shiller1992computation,pham2014general},
  which results in an overall complexity of $O(m^2N)$.

\item \emph{Convex optimization approach:} This approach formulates
  the TOPP problem as a single large convex optimization program with
  $O(N)$ variables and $O(mN)$ inequality constraints. In the fastest
  implementation we know of, the author solves the convex optimization
  problem by solving a sequence of linear programs (SLP) with the same
  number of variables and inequalities~\cite{hauser2014}. Thus, the
  time complexity of this approach is $O(KmN^3)$, where $K$ is the
  number of SLP iterations.
\end{itemize}

This analysis shows that TOPP-RA has the best theoretical
complexity. The next section experimentally assesses this observation.


\section{Experiments}
\label{sec:experiments}

We implements TOPP-RA in Python on a machine running Ubuntu with a
Intel i7-4770(8) 3.9GHz CPU and 8Gb RAM\@. To solve the LPs we use the
Python interface of the solver
\texttt{qpOASES}~\cite{Ferreau2014}. The implementation and test cases
are available at~\url{https://github.com/hungpham2511/toppra}.

\subsection{Experiment 1: Pure joint velocity and acceleration bounds}
\label{sec:time-optimal-path}

In this experiment, we compare TOPP-RA against TOPP-NI -- the fastest
known implementation of TOPP, which is based on the Numerical
Integration approach~\cite{pham2014general}. For simplicity, we
consider pure joint velocity and acceleration bounds, which involve
the same difficulty as any other types of kinodynamic constraints, as
far as TOPP is concerned.

\subsubsection{Effect of the number of constraint inequalities}

We considered random geometric paths with varying degrees of freedom
$n\in[2,60]$. Each path was generated as follows: we sampled $5$
random waypoints and interpolated smooth geometric paths using cubic
splines. For each path, velocity and acceleration bounds were also
randomly chosen such that the bounds contain zero. This ensures that
all generated instances are feasible. Each problem instance thus has
$m=2n+2$ constraint inequalities: $2n$ inequalities corresponding to
acceleration bounds (no pruning was applied, contrary
to~\cite{hauser2014}) and $2$ inequalities corresponding to velocity
bounds (the joint velocity bounds could be immediately pruned into one
lower and one upper bound on $\dot s$). According to the complexity
analysis of Section~\ref{sec:runn-time-analys}, we consider the number
of inequalities, rather than the degree of freedom, as independent
variable. Finally, the discretization grid size was chosen as $N=500$.

Fig.~\ref{fig:correct} shows the time-parameterizations and the
resulting trajectories produced by TOPP-RA and TOPP-NI on an instance with
$(n=6, m=14)$. One can observe that the two algorithms produced
virtually identical results, hinting at the correctness of TOPP-RA.

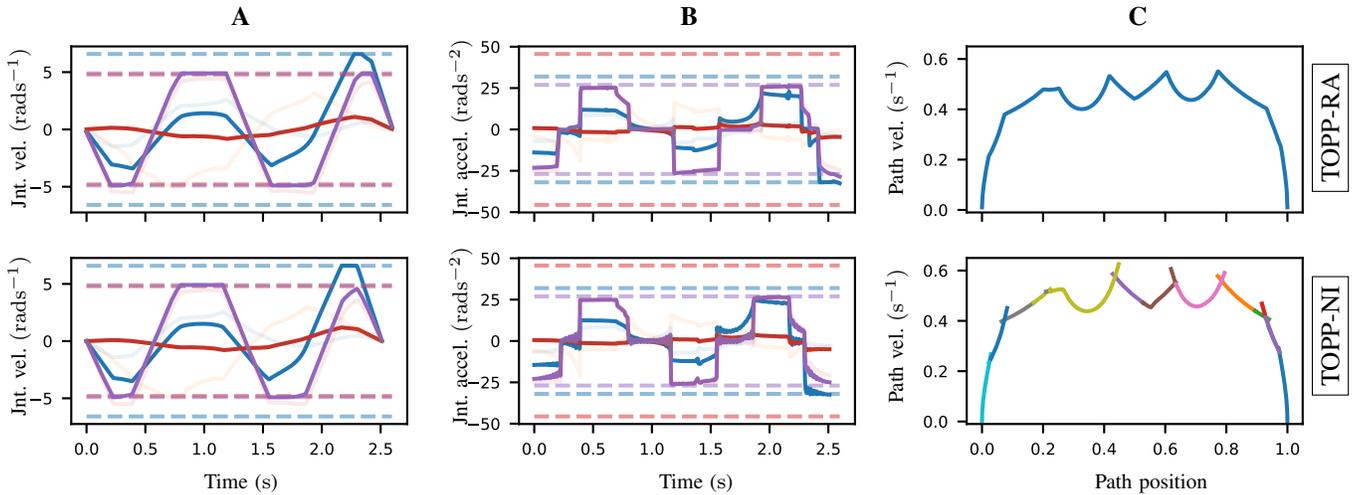
\begin{figure*}[ht]
  \centering
  \begin{tikzpicture}
    \node[anchor=south west,inner sep=0] (image) at (0,0)
    {\includegraphics[]{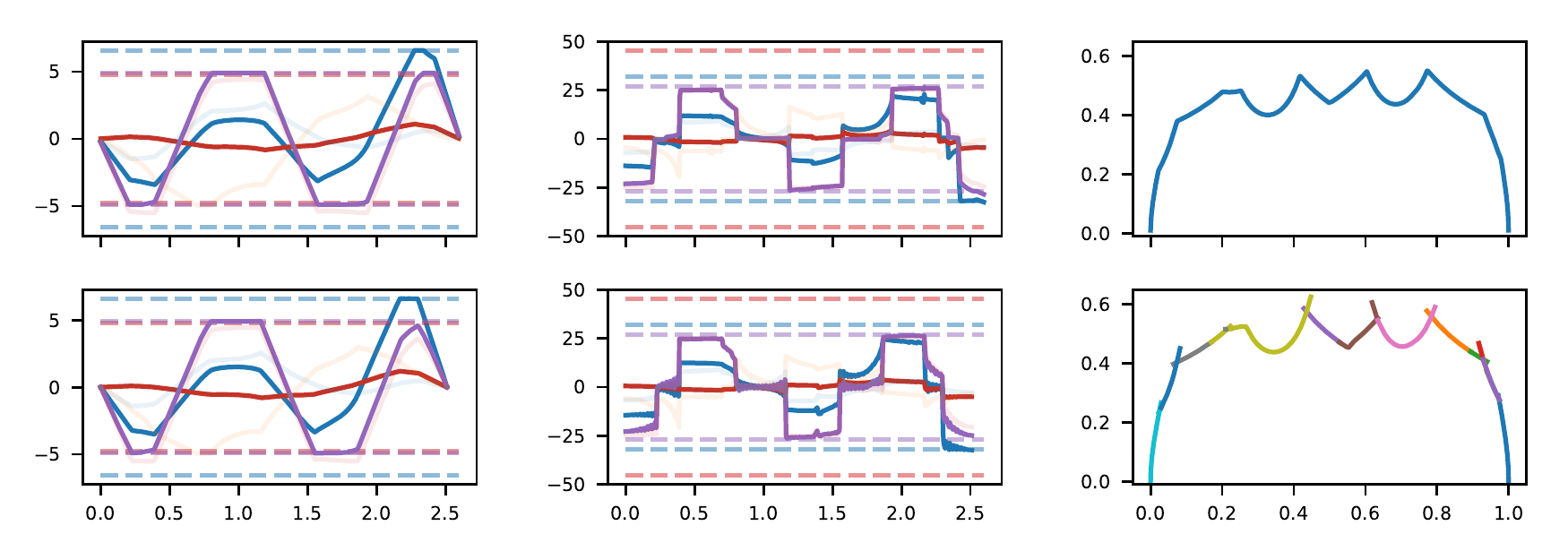}};
    \begin{scope}[x={(image.south east)},y={(image.north west)}]
      \node[anchor=south] at (0.18,  0.95) {\textbf{A}};
      \node[anchor=south] at (0.515, 0.95) {\textbf{B}};
      \node[anchor=south] at (0.85,  0.95) {\textbf{C}};
      \node[axes label, rotate=90, anchor=south] at (.030, 0.755)
      {Jnt.\ vel.\ $(\SI{}{rad s^{-1}})$};
      \node[axes label, rotate=90, anchor=south] at (.030, 0.32)
      {Jnt.\ vel.\ $(\SI{}{rad s^{-1}})$};
      \node[axes label, rotate=90, anchor=south] at (.36, 0.755)
      {Jnt.\ accel.\ $(\SI{}{rad s^{-2}})$};
      \node[axes label, rotate=90, anchor=south] at (.36, 0.32)
      {Jnt.\ accel.\ $(\SI{}{rad s^{-2}})$};
      \node[axes label, rotate=90] at (.67, 0.32)  {Path vel. $(\SI{}{s^{-1}})$};
      \node[axes label, rotate=90] at (.67, 0.755) {Path vel. $(\SI{}{s^{-1}})$};
      \node[axes label, anchor=north] at (.18, 0.05) {Time ($\SI{}{s}$)};
      \node[axes label, anchor=north] at (.52, 0.05) {Time ($\SI{}{s}$)};
      \node[axes label, anchor=north] at (.86, 0.05) {Path position};
      \node[rotate=90, anchor=north, draw] at (0.98, 0.755) {TOPP-RA};
      \node[rotate=90, anchor=north, draw] at (0.98, 0.32) {TOPP-NI};
    \end{scope}
  \end{tikzpicture}
  \caption{\label{fig:correct} Time-optimal parameterization of a
    6-dof path under velocity and acceleration bounds ($m=14$
    constraint inequalities and $N=500$ grid points).  TOPP-RA and
    TOPP-NI produce virtually identical results. \textbf{(A)}: joint
    velocities. \textbf{(B)}: joint accelerations. \textbf{(C)}:
    velocity profiles in the $(s,\dot s)$ plane. Note the small
    chattering in the joint accelerations produced by TOPP-NI, which
    is an artifact of the integration process. This chattering is
    absent from the TOPP-RA profiles.}
\end{figure*}

Fig.~\ref{fig:comp-solve} shows the computation time for TOPP-RA and
TOPP-NI, excluding the ``setup'' and ``extract trajectory'' steps
(which takes much longer in TOPP-NI than in TOPP-RA). The experimental
results confirm our theoretical analysis in that the complexity of
TOPP-RA is in linear in $m$ while that of TOPP-NI is quadratic in
$m$. In terms of actual computation time, TOPP-RA becomes faster than
TOPP-NI as soon as $m\geq 22$.  Table~\ref{tab:computation-time}
reports the different components of the computation time.



\begin{figure}[htp]
  \centering
  \begin{tikzpicture}
    \node[anchor=south west,inner sep=0] (image) at (0,0)
    {\includegraphics[]{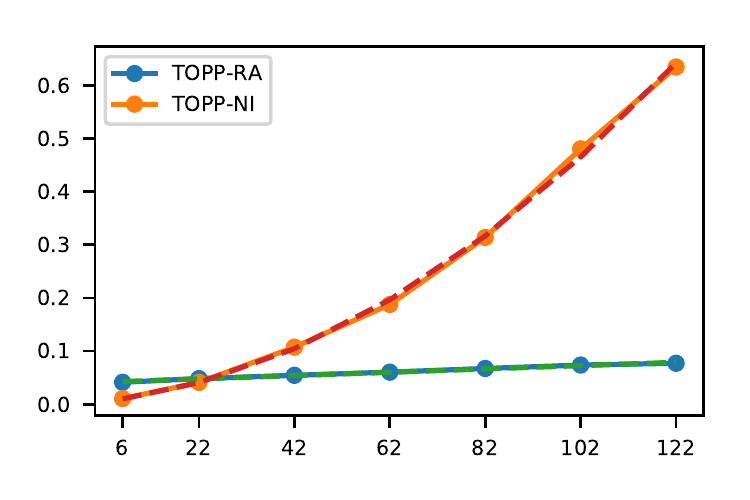}};
    \begin{scope}[x={(image.south east)},y={(image.north west)}]
      \node[axes label] at (0.53, 0) {No.\ of inequalities};
      \node[rotate=90, style=axes label, text width=4cm] at (-0.0, 0.53)
         {Avg.\ solve time \\ (per grid point) $(\SI{}{ms})$};
    \end{scope}
  \end{tikzpicture}
  \caption{\label{fig:comp-solve} Computation time of TOPP-RA (solid
    blue) and TOPP-NI (solid orange), excluding the ``setup'' and
    ``extract trajectory'' steps, as a function of the number of
    constraint inequalities. Confirming our theoretical complexity
    analysis, the complexity of TOPP-RA is linear in the number of
    constraint inequalities $m$ (linear fit in dashed green), while
    that of TOPP-NI is quadratic in $m$ (quadratic fit in dashed
    red). In terms of actual computation time, TOPP-RA becomes faster
    than TOPP-NI as soon as $m\geq 22$.}
\end{figure}


\begin{table}[htp]
  \caption{Breakdown of TOPP-RA and TOPP-NI total computation time to
    parameterize a path discretized with $N=500$ grid points, subject
    to $m=30$ inequalities.}
  \label{tab:computation-time}
  \centering
  \begin{tabular}{lllllrr}
    \toprule
    & & \multicolumn{5}{c}{Time ($\SI{}{ms}$)} \\
    \cmidrule{3-7}
    \multicolumn{2}{l}{}               &  \multicolumn{2}{l}{TOPP-RA} & \multicolumn{2}{l}{TOPP-RA-intp} & TOPP-NI	      \\
    \midrule
    \multicolumn{2}{l}{setup}          &  \multicolumn{2}{l}{1.0}     & \multicolumn{2}{l}{1.5}          & 123.6  \\
    \multicolumn{2}{l}{solve TOPP}     &  \multicolumn{2}{l}{26.1}    & \multicolumn{2}{l}{29.1}          & 28.3    \\
                         & backward pass&  & 16.5&   & 19.9&   \\
                         & forward pass &  & 9.6&   & 9.2&    \\
    \multicolumn{2}{l}{extract trajectory}& \multicolumn{2}{l}{2.7} & \multicolumn{2}{l}{2.7} & 303.4\\
    \cmidrule{3-7}
    \multicolumn{2}{l}{total} & \multicolumn{2}{l}{29.8} & \multicolumn{2}{l}{33.3} & 455.3\\
    \bottomrule
  \end{tabular}
\end{table}

Perhaps even more importantly than mere computation time, TOPP-RA was
extremely robust: it maintained $100\%$ success rate over all
instances, while TOPP-NI struggled with instances with many inequality
constraints ($m\geq 40$), see Fig.~\ref{fig:success-rate}. Since all
TOPP instances were feasible, an algorithm failed when it did not
return a correct parameterization.

\begin{figure}[htp]
  \centering
  \begin{tikzpicture}
    \node[anchor=south west,inner sep=0] (image) at (0,0)
    {\includegraphics[]{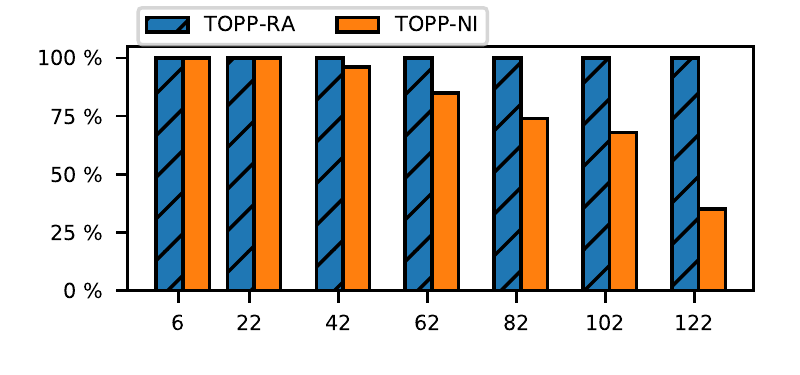}};
    \begin{scope}[x={(image.south east)},y={(image.north west)}]
      \node [axes label, rotate=90, anchor=south] at (0.05,0.55) {Success rate};
      \node [axes label] at (0.5, 0.03) {No.\ of inequalities};
    \end{scope}
  \end{tikzpicture}
  \caption{Success rate for TOPP-RA and TOPP-NI\@. TOPP-RA enjoys
    consistently $100\%$ success rate while TOPP-NI reports failure
    for more complex problem instances ($m\geq 40$).}
  \label{fig:success-rate}
\end{figure}


\subsubsection{Effect of discretization grid size}
\label{sec:experiment-3:-}

Grid size (or its inverse, discretization time step) is an important
parameter for both TOPP-RA and TOPP-NI as it affects running time,
success rate and solution quality, as measured by constraint
satisfaction error and sub-optimality. Here, we assess the effect of
grid size on \emph{success rate} and \emph{solution quality}. Remark
that, based on our complexity analysis in
Section~\ref{sec:runn-time-analys}, running time depends linearly on
grid size in both algorithms.

In addition to TOPP-RA and TOPP-NI, we considered TOPP-RA-intp. This
variant of TOPP-RA employs the first-order interpolation scheme (see
Appendix~\ref{sec:discr-scheme-error}) to discretize the constraints,
instead of the collocation scheme introduced in
Section~\ref{sec:discr-path-dynam}.


We considered different grid sizes $N\in[100, 1000]$. For each grid
size, we generated and solved $100$ random parameterization instances;
each instance consists of a random path with $n=14$ subject to random
kinematic constraints, as in the previous
experiment. Fig.~\ref{fig:grid_size}-\textbf{A} shows success rates
versus grid sizes. One can observe that TOPP-RA and TOPP-RA-intp
maintained $100\%$ success rate across all grid sizes, while TOPP-NI
reported two failures at $N=100$ and $N=1000$.

\begin{figure}[htp]
  \centering
  \begin{tikzpicture}
    \node[anchor=south west, inner sep=0] (bottom) at (0, 0)
    {\includegraphics[width=0.46\textwidth]{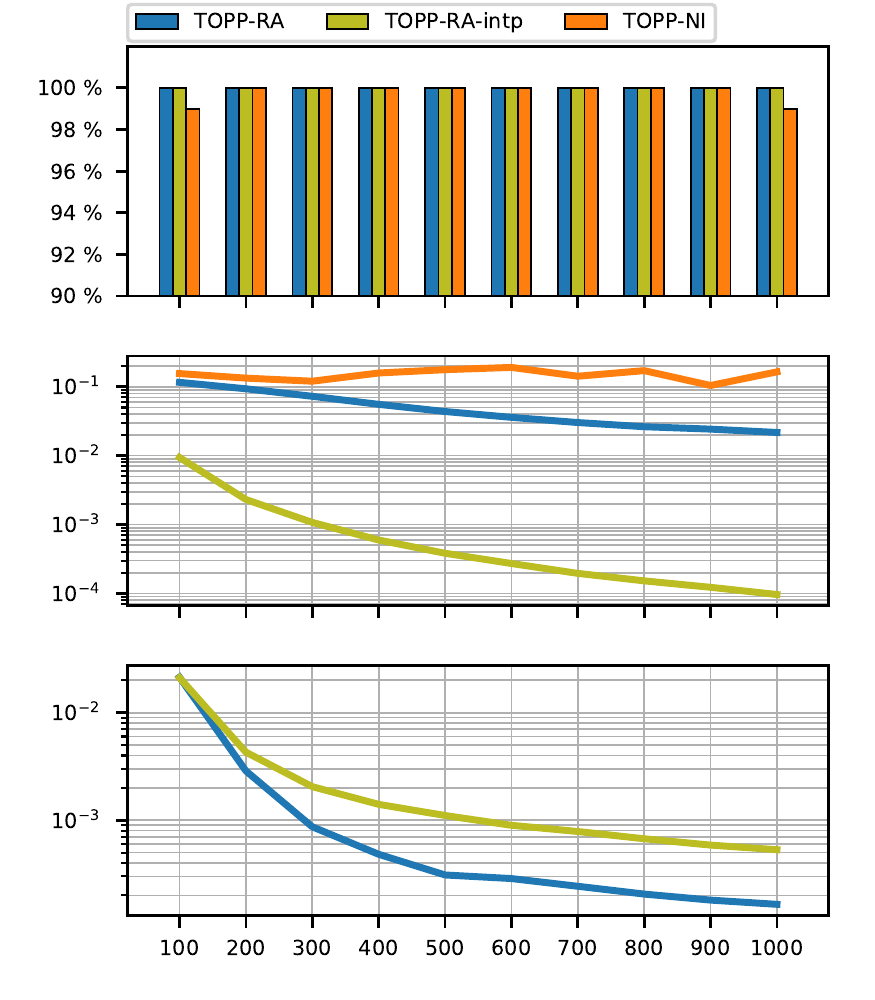}};
    \node[axes label, anchor=north] at (4.5, 0.4) {Grid size};
    \node[axes label, anchor=south, rotate=90] at (.5, 7.7) {Success rate (\%)};
    \node[axes label, anchor=south, rotate=90] at (.5, 4.8)
    {Relative constraint sat.\ error};
    \node[axes label, anchor=south, rotate=90] at (.5, 2.0)
    {$|J^{*} - J^{\dagger}|$};
    \node[anchor=south] at (7.8, 9.1) {\textbf{A}};
    \node[anchor=south] at (7.8, 6.15) {\textbf{B}};
    \node[anchor=south] at (7.8, 3.2) {\textbf{C}};
  \end{tikzpicture}
  \caption{\textbf{(A)}: effect of grid size on success
    rate. \textbf{(B)}: effect of grid size on relative constraint
    satisfaction error, defined as the ratio between the error and the
    respective bound. TOPP-RA-intp returned solutions that are orders
    of magnitude better than TOPP-RA and TOPP-NI. \textbf{(C)}: effect
    of grid size on difference between the average solution's cost and
    the optimal cost, which is approximated by solving TOPP-RA-intp
    with $N=10000$.  Solutions produced by TOPP-RA-intp had higher
    costs than those produced by TOPP-RA as those instances were more highly
    constrainted.  }
  \label{fig:grid_size}
\end{figure}
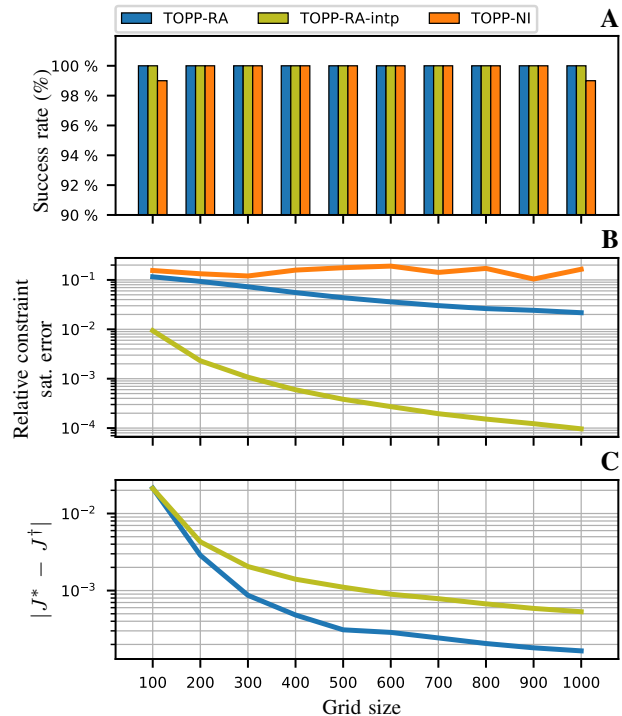


Next, to measure the effect of grid size on solution quality, we
looked at the \emph{relative greatest constraint satisfaction errors},
defined as the ratio between the errors, whose definition is given in
Appendix~\ref{sec:error-analysis}, and the respective bounds. For each
instance, we sampled the resulting trajectories at $\SI{1}{ms}$ and
computed the greatest constraint satisfaction errors by comparing the
sampled joint accelerations and velocities to their respective
bounds. Then, we averaged instances with the same grid size to obtain
the average error for each $N$.

Fig.~\ref{fig:grid_size}-\textbf{B} shows the average relative
greatest constraint satisfaction errors of the three algorithms with
respect to grid size. One can observe that TOPP-RA and TOPP-NI have
constraint satisfaction errors of the same order of magnitude for
$N < 500$, while TOPP-RA demonstrates better quality for $N \geq 500$.
TOPP-RA-intp produces solutions with much higher quality. This result
confirms our error analysis of different discretization schemes in
Appendix~\ref{sec:discr-scheme-error} and demonstrates that the
interpolation discretization scheme is better than the collocation
scheme whenever solution quality is concerned.

Fig.~\ref{fig:grid_size}-\textbf{C} shows the average difference
between the costs of solutions returned by TOPP-RA and TOPP-RA-intp
with the \emph{true} optimal cost, which was  approximated by
running TOPP-RA-intp with grid size $N=10000$. One can observe
that both algorithms are asymptotically optimal. Even more
importantly, the differences are relatively small: even at the coarse grid
size of $N=100$, the difference is only $10^{-2}\SI{}{sec}$.

\subsection{Experiment 2: Legged robot in multi-contact}
\label{sec:topp-subject-non}

Here we consider the time-parameterization problem for a 50-dof legged
robot in multi-contact under joint torque bounds and linearized
friction cone constraints.

\subsubsection{Formulation}

We now give a brief description of our formulation, for more details,
refer to~\cite{hauser2014,pham2015time}. Let $\vec w_{i}$ denote the
net contact wrench (force-torque pair) exerted on the robot by the
$i$-th contact at point $\vec p_{i}$. Using the linearized friction
cone, one obtains the set of feasible wrenches as a polyhedral cone
\begin{equation*}
 \{ \vec w_i \mid \vec F_i \vec w_i \leq 0\},
\end{equation*}
for some matrix $\vec F_{i}$. This matrix can be found using the Cone
Double Description method~\cite{fukuda1996double, Caron2017}.
Combining with the equation governing rigid-body dynamics, we obtain
the full dynamic feasibility constraint as follow
\begin{equation*}
  \begin{aligned}
    \vec M(\vec q)\ddot{\vec q}+&\dot{\vec q}^\top\vec C(\vec
    q)\dot{\vec q}+\vec g(\vec q)=\bftau + \sum_{i=1, 2} \vec
    J_{i}(\vec q)^{\top}\vec w_{i}, \\
    &\vec F_i \vec w_i \leq 0, \\
    &\bftau_{\min} \leq \bftau \leq \tau_{\max},\\
  \end{aligned}
\end{equation*}
where $\vec J_i(\vec q)$ is the wrench Jacobian.  The convex set
$\cC(\vec q)$ in Eq.~\eqref{eq:general-form} can now be identified as a
multi-dimensional polyhedron.

We considered a simple swaying motion: the robot stands with both feet
lie flat on two uneven steps and shift its body back and forth, see
Fig.~\ref{fig:humanoid}.  The coefficient of friction was set to
$\mu = 0.5$. Start and end path velocities were set to
zero. Discretization grid size was $N=100$. The number of constraint
inequalities was $m=242$.

\tikzstyle text box=[axes label, text width=1.3cm,
draw, rectangle, text=black, fill=white, anchor=north east]

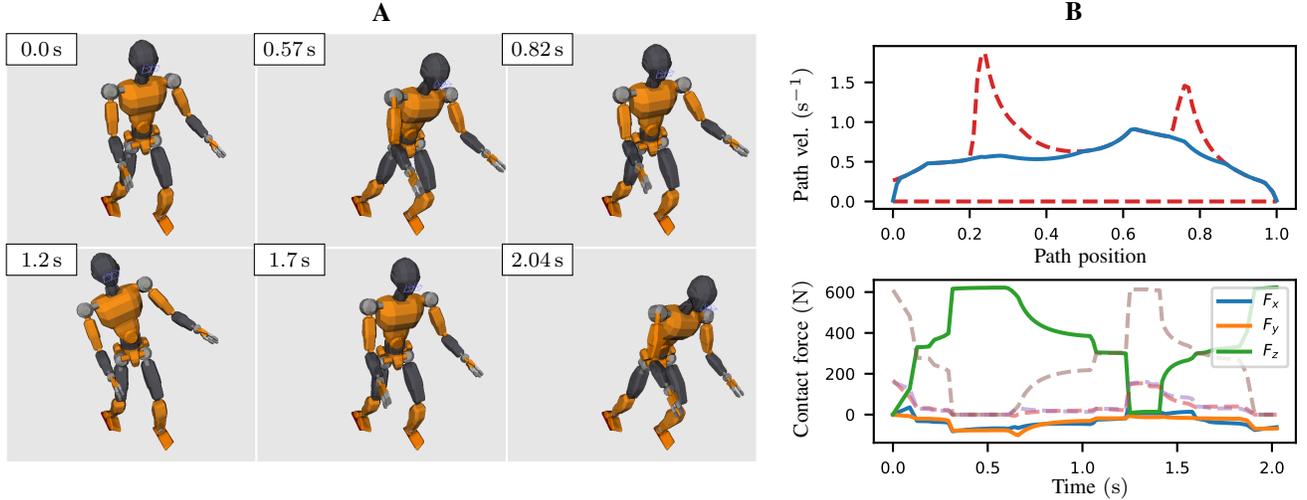
\begin{figure*}[ht]
  \centering
  \begin{tikzpicture}
    \node[anchor=south west,inner sep=0] (figure)  at (10.5,0)
    {\includegraphics[]{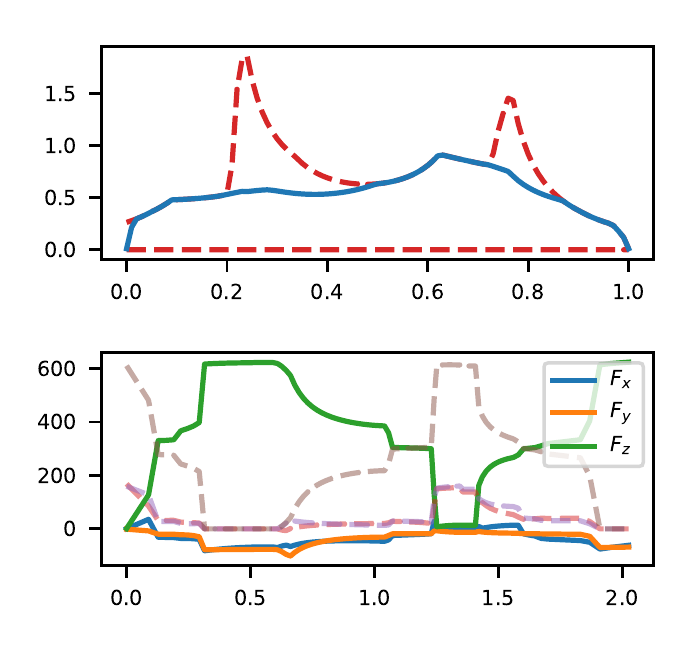}};
    \node[anchor=south west,inner sep=0] (image) at (0, 0.6){
      \includegraphics[width=0.55\textwidth]{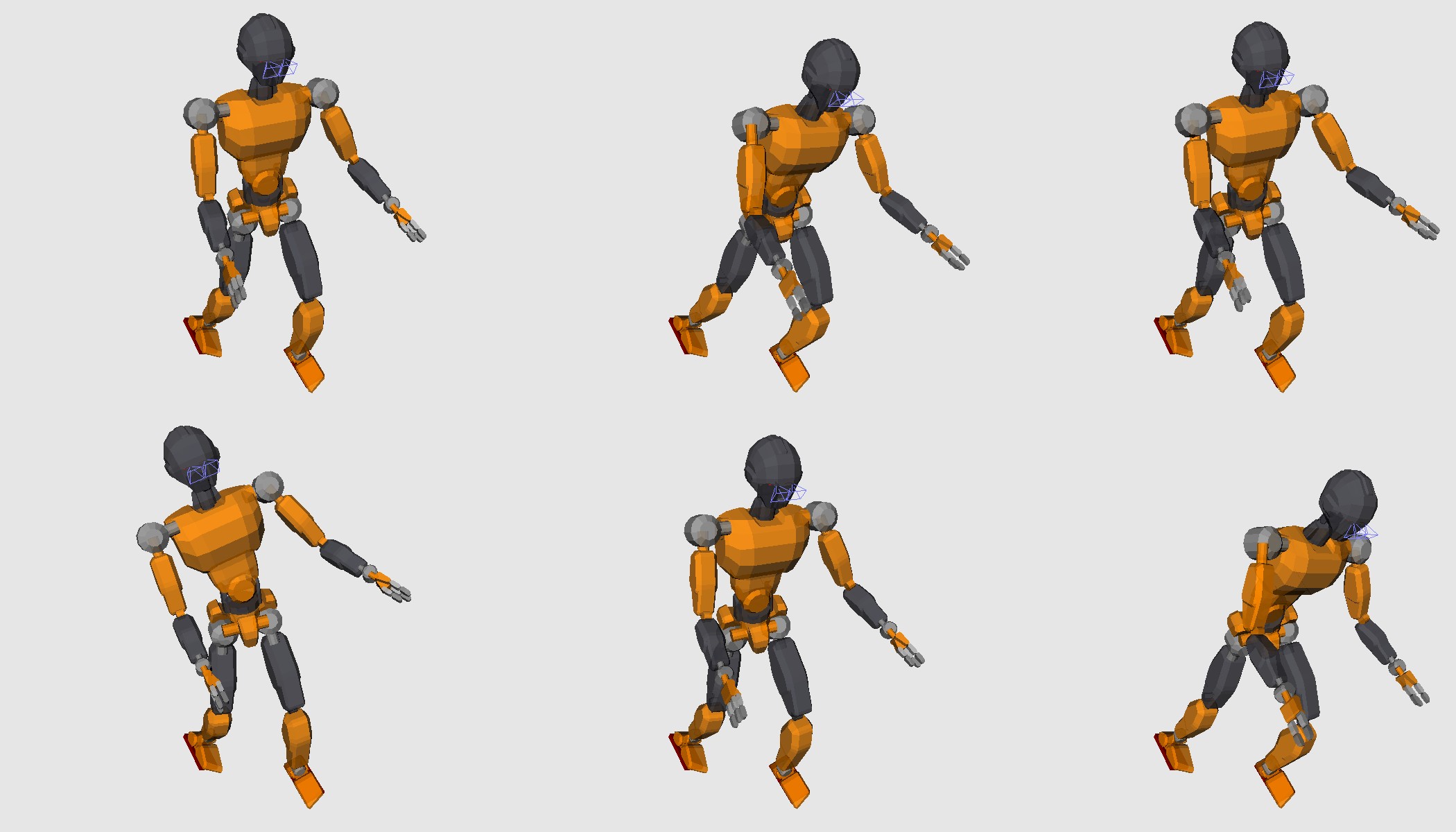}};
    \node at (5,    6.6) {\textbf{A}};
    \node at (14.2,   6.6) {\textbf{B}};
    \begin{scope}[shift={(figure.south west)},
                  x={(figure.south east)},
                  y={(figure.north west)}]
    \node[axes label, text width=2.5cm,
           rotate=90, anchor=south] at (0.05, 0.3)
         {Contact force  $(\SI{}{N})$};
    \node[axes label, rotate=90, anchor=south] at (0.05,  0.75)
         {Path vel. $(\SI{}{s^{-1}})$};
    \node[axes label, anchor=north] at (0.55,  0.08) {Time $(\SI{}{s})$};
    \node[axes label, anchor=north] at (0.55,  0.54) {Path position};
    \end{scope}
    \begin{scope}[shift={(image.south west)},
                  x={(image.south east)},
                  y={(image.north west)}]
      \draw[xstep=1/3, ystep=1/2, color=white, thick] (0, 0) grid (1, 1);

      \node[text box, anchor=north west, text width=0.7cm] at (0,   1)    {$\SI{0.  }{s}$};
      \node[text box, anchor=north west, text width=0.7cm] at (0.33,  1)  {$\SI{0.57}{s}$};
      \node[text box, anchor=north west, text width=0.7cm] at (0.66,  1)  {$\SI{0.82}{s}$};
      \node[text box, anchor=north west, text width=0.7cm] at (0,     0.51){$\SI{1.2 }{s}$};
      \node[text box, anchor=north west, text width=0.7cm] at (0.33,  0.51){$\SI{1.7 }{s}$};
      \node[text box, anchor=north west, text width=0.7cm] at (0.66,  0.51){$\SI{2.04}{s}$};

    \end{scope}
  \end{tikzpicture}
  \caption{Time-parameterization of a legged robot trajectory under
    joint torque bounds and multi-contact friction
    constraints. \textbf{(A)}: Snapshots of the retimed
    motion. \textbf{(B)}: Optimal velocity profile computed by TOPP-RA
    (blue) and upper and lower limits of the controllable sets (dashed
    red).  \textbf{(C)}: The optimal joint torques and contact forces
    are obtained ``for free'' as slack variables of the optimization
    programs solved in the forward pass. Net contact forces for the
    left foot are shown in colors and those for the right foot are
    shown in transparent lines.}
  \label{fig:humanoid}
\end{figure*}

\subsubsection{Results}

Excluding computation of dynamic quantities, TOPP-RA took
$\SI{267}{ms}$ to solve for the time-optimal path parameterization on
our computer.  The final parameterization is shown in
Fig.~\ref{fig:humanoid} and computation time is presented in
Table~\ref{tab:computation-time-poly}.

Compared to TOPP-NI and TOPP-CO, TOPP-RA had significantly better
computation time, chiefly because both existing methods require an
expensive polytopic projection step.  Indeed, \cite{hauser2014}
reported projection time of $\SI{2.4}{s}$ for a similar sized problem,
which is significantly more expensive than TOPP-RA computation
time. Notice that in~\cite{hauser2014}, computing the parameterization
takes an addition $\SI{2.46}{s}$ which leads to a total computation time
of $\SI{4.86}{s}$.

To make a more accurate comparison, we implement the following
pipeline on our computer to solve the same problem~\cite{pham2015time}
\begin{enumerate}
\item project the constraint polyhedron $\cC_i$ onto the path using
  Bretl's polygon recursive expansion algorithm~\cite{bretl2008tro};
\item parameterize the resulting problem using TOPP-NI.
\end{enumerate}
This pipeline turned out to be much slower than TOPP-RA\@.  We found
that the number of LPs the projection step solved is nearly 8 times
more than the number of LPs solved by TOPP-RA (which is fixed at
$3N = 300$).  For a more detailed comparison of computation time and
parameters of the LPs, refer to Table~\ref{tab:computation-time-poly}.

\subsubsection{Obtaining joint torques and contact forces ``for free''}
\label{sec:obta-joint-torq}

Another interesting feature of TOPP-RA is that the algorithm can
optimize and obtain joint torques and contact forces ``for free''
without additional processing. Concretely, since joint torques and
contact forces are slack variables, one can simply store the optimal
slack variable at each step and obtain a trajectory of feasible
forces. To optimize the forces, we can solve the following quadratic
program (QP) at the $i$-th step of the forward pass
\begin{equation*}
  \begin{aligned}
    \min & \quad - u + \epsilon \|(\vec w, \bftau)\|_2^2\\
    \mathrm{s.t.} & \quad x = x_i \\
         & \quad (u, x)  \in \Omega_i\\
         & \quad x + 2 \Delta_i u \in K_{i+1},
  \end{aligned}
\end{equation*}
where $\epsilon$ is a positive scalar.  Figure~\ref{fig:humanoid}'s
lower plot shows computed contact wrench for the left leg.  We note
that both existing approaches, TOPP-NI and TOPP-CO are not able to
produce joint torques and contact forces readily as they ``flatten''
the constraint polygon in the projection step.

In fact, the above formulation suggests that time-optimality is simply
a specific objective cost function (linear) of the more general family
of quadratic objectives. Therefore, one can in principle depart from
time-optimality in favor of more realistic objective such as
minimizing torque while maintaining a certain nominal velocity
$x_{\rm{norm}}$ as follow
\begin{equation*}
  \begin{aligned}
    \min & \quad \|x_i + 2 \Delta_i u - x_{\mathrm{norm}} \|_2^2 + \epsilon \|(\vec w, \bftau)\|_2^2 \\
    \mathrm{s.t.} & \quad x = x_i \\
         & \quad (u, x)  \in \Omega_i\\
         & \quad x + 2 \Delta_i u \in K_{i+1}.
  \end{aligned}
\end{equation*}

Finally, we observed that the choice of path discretization scheme has
noticeable effects on both computational cost and quality of the
result.  In general, TOPP-RA-intp produced smoother trajectories and
better (lower) constraint satisfaction error at the cost of longer
computation time. On the other hand, TOPP-RA was faster but produced
trajectories with jitters~\footnote{Our experiments show that
  singularities do not cause parameterization failures for TOPP-RA and
  the jitters can usually be removed easily. One possible method is to
  use cubic splines to smooth the velocity profile locally around the
  jitters.} near dynamic singularities~\cite{pham2014general} and had
worse (higher) constraint satisfaction error.

\begin{table}[ht]
  \caption{Computation time ($\SI{}{ms}$) and internal parameters
    comparison between TOPP-RA, TOPP-NI and TOPP-RA-intp (first-order
    interpolation) in Experiment 2.}
\label{tab:computation-time-poly}
\centering
\begin{tabular}{lrrr}
  \toprule
  & TOPP-RA         &  TOPP-RA-intp  & TOPP-NI            	      \\
  \midrule
  & \multicolumn{3}{c}{Time ($\SI{}{ms}$)}\\
  \cmidrule{2-4}
  comp.\ dynamic quantities  &181.6  & 193.6   &281.6  \\
  polytopic projection      & 0.0   &    0.0  &3671.8 \\
  solve TOPP                & 267.0 & 1619.0  & 335.0 \\
  extract trajectory        &  3.0  &     3.0 &  210.0 \\
  \cmidrule{2-4}
  total                     & 451.6 &  1815.6 & 4497.8 \\
  \midrule
  & \multicolumn{3}{c}{Parameters}\\
  \cmidrule{2-4}
  joint torques /        &yes   &  yes  & no  \\
  contact forces avail.  &      &       &     \\
  No.\ of LP(s) solved    & 300  & 300   & 2110\\
  No.\ of variables       &64    & 126   & 64  \\
  No.\ of constraints     & 242  &  476  & 242 \\
  Constraints sat.\ error & $O(\Delta)$ & $O(\Delta^2)$ & $O(\Delta)$  \\
  \bottomrule
\end{tabular}
\end{table}

\section{Additional benefits of TOPP by Reachability Analysis}
\label{sec:optimal-path-policy}

We now elaborate on the additional benefits provided by the reachability
analysis approach to TOPP.

\subsection{Admissible Velocity Propagation}
\label{sec:avp}

Admissible Velocity Propagation (AVP) is a recent concept for
kinodynamic motion planning~\cite{pham2017admissible}. Specifically,
given a path and an initial interval of velocities, AVP returns
exactly the interval of all the velocities the system can reach after
traversing the path while respecting the system kinodynamic
constraints. Combined with existing \emph{geometric} path planners,
such as RRT~\cite{kuffner2000rrt}, this can be advantageously used for
\emph{kinodynamic} motion planning: at each tree extension in the
configuration space, AVP can be used to guarantee the eventual
existence of admissible path parameterizations.

Suppose that the initial velocity interval is $\bbI_0$. It can be
immediately seen that, what is computed by AVP is exactly the
reachable set $\calL_N(\bbI_0)$
(cf. Section~\ref{sec:reachable-sets}). Furthermore, what is computed
by AVP-Backward~\cite{lertkultanon2014dynamic} given a desired final
velocity interval $\bbI_N$ is exactly the controllable set
$\calK_0(\bbI_N)$ (cf. Section~\ref{sec:controllable-sets}). In terms
of complexity, $\calR_N(\bbI_0)$ and $\calK_0(\bbI_N)$ can be found by
solving respectively $2N$ and $2N$ LPs. We have thus re-derived the
concepts of AVP at no cost.

\subsection{Robustness to parametric uncertainty}
\label{sec:robust-moti-plann}

In most works dedicated to TOPP, including the development of the
present paper up to this point, the parameters appearing in the
dynamics equations and in the constraints are supposed to be exactly
known. In reality, those parameters, which include inertia
matrices or payloads in robot manipulators, or feet positions or
friction coefficients in legged robots, are only known up to some
precision. An admissible parameterization for the nominal values of
the parameters might not be admissible for the actual values, and the
probability of constraints violation is even higher in the \emph{optimal}
parameterization, which saturates at least one constraint at any
moment in time.

TOPP-RA provides a natural way to handle parametric
uncertainty. Assume that the constraints appear in the following form
\begin{equation}
  \label{eq:unc-cnst}
  \begin{aligned}
    &\mathbf{a}_i u + \mathbf{b}_i x+ \mathbf{c}_i \in \cC_i,\\
    &\forall (\vec a_i, \vec b_i, \vec c_i, \cC_i) \in \calE_i,
  \end{aligned}
\end{equation}
where $\calE_i$ contains all the possible values that the parameters
might take at path position $s_i$.

\begin{remark}
  Consider for instance the manipulator with torque bounds of
  equation~\eqref{eq:manip}. Suppose that, at path position $i$, the
  inertia matrix is uncertain, i.e., that it might take any values
  $\vec M_i \in B(\vec M^\textrm{nominal}_i,\epsilon)$, where
  $B(\vec M^\textrm{nominal}_i,\epsilon)$ denotes the ball of radius
  $\epsilon$ centered around $\vec M^\textrm{nominal}_i$ for the max
  norm. Then, the first component of $\calE_i$ is given by
  $\{\vec M_i\vec q'(s_i) \mid \vec M_i \in B(\vec
  M^\textrm{nominal}_i,\epsilon)\}$,
  which is a convex set.

  In legged robots, uncertainties on feet positions or on friction
  coefficients can be encoded into a ``set of sets'', in which $\cC_i$
  can take values.  
\end{remark}

TOPP-RA can handle this situation by suitably modifying its two
passes. Before presenting the modifications, we first give some
definitions.  Denote the $i$-stage set of \emph{robust admissible}
control-state pairs by
\begin{equation*}
  \widehat{\Omega}_{i} := \{(u,x) \mid \textrm{Eq.~\eqref{eq:unc-cnst} holds}\}.
\end{equation*}
The sets of robust admissible states $\widehat{\calX}_i$ and robust
admissible controls $\widehat{\calU}_i(x)$ can be defined as in
Section~\ref{sec:useful-definitions}.

In the backward pass, TOPP-RA computes the \emph{robust controllable
  sets}, whose definition is given below. 


\begin{definition}[$i$-stage \emph{robust} controllable set]
  Consider a set of desired ending states $\bbI_N$. The
  \emph{$i$-stage robust controllable set}
  $\widehat{\mathcal{K}}_i(\bbI_N)$ is the set states
  $x\in\widehat{\calX_i}$ such that there exists a state
  $x_N\in \bbI_N$ and a sequence of \emph{robust admissible controls}
  $u_i,\dots, u_{N-1}$ that steers the system from $x$ to $x_{N}$.
\end{definition}

To compute the robust controllable sets, one needs the robust
one-step set.

\begin{definition}[\emph{Robust} one-step set]
  Consider a set of states $\bbI$. The \emph{robust one-step set}
  $\widehat{\calQ}_i(\bbI)$ is the set of states
  $x\in\widehat{\calX}_{i}$ such that there exists a state
  $\tilde{x}\in\bbI$ and a robust admissible control
  $u\in\widehat{\calU}_i(x)$ that steers the system from $x$ to
  $\tilde{x}$.
\end{definition}

Finally, in the forward pass, the algorithm selected the
\emph{greatest} robust admissible control at each stage.

\begin{remark}
  Computing the robust one-step set and the greatest robust admissible
  control involves solving LPs with uncertain constraints of the
  form~\eqref{eq:unc-cnst}. In general, these constraints may contain
  hundreds of inequalities, making them difficult to handle by generic
  methods. In the mathematical optimization literature, they are known
  as ``Robust Linear Programs'', and specific methods have been
  developed to handle them efficiently, when the robust constraints
  are~\cite{ben2001lectures}
  \begin{enumerate}
  \item polyhedra;
  \item ellipsoids;
  \item Conic Quadratic re-presentable (CQr) sets.
  \end{enumerate}
  The first case can be treated as normal LPs with appropriate slack
  variables, while the last two cases are explicit Conic Quadratic
  Program (CQP). For more information on this conversion, refer to the
  first and second chapters of~\cite{ben2001lectures}.
\end{remark}


\section{Conclusion}
\label{sec:discussion}

We have presented a new approach to solve the Time-Optimal Path
Parameterization (TOPP) problem based on Reachability Analysis
(TOPP-RA). The key insight is to compute, in a first pass, the sets of
controllable states, for which admissible controls allowing to reach
the goal are guaranteed to exist. Time-optimality can then be
obtained, in a second pass, by a simple greedy strategy. We have
shown, through theoretical analyses and extensive experiments, that
the proposed algorithm is extremely robust (100\% success rate), is
competitive in terms of computation time as compared to the fastest
known TOPP implementation~\cite{pham2014general} and produces
solutions with high quality. Finally, the new approach yields
additional benefits: no need for polytopic projection in the
redundantly-actuated case, Admissible Velocity Projection, and
robustness to parameter uncertainty.

A recognized disadvantage of the classical TOPP formulation is that
the time-optimal trajectory contains hard acceleration switches,
corresponding to infinite jerks. Solving TOPP subject to jerk bounds,
however, is not possible using the CO-based approach as the problem
becomes non-convex~\cite{verscheure2008practical}. Some prior works
proposed to either extend the NI-based
approach~\cite{tarkiainen1993time,HungPham2017} or to represent the
parameterization as a spline and optimize directly over the parameter
space~\cite{costantinescu2000smooth,oberherber2015successive}. Exploring
how Reachability Analysis can be extended to handle jerk bounds is
another direction of our future research.

Similar to the CO-based approach, Reachability Analysis can only be
applied to instances with convex
constraints~\cite{verscheure2008practical}. Yet in practice, it is
often desirable to consider in addition non-convex constraints, such
as joint torque bounds with viscous friction effect.  Extending
Reachability Analysis to handle non-convex constraints is another
important research question.

\subsection*{Acknowledgment}

This work was partially supported by grant ATMRI:2014-R6-PHAM (awarded
by NTU and the Civil Aviation Authority of Singapore) and by the
Medium-Sized Centre funding scheme (awarded by the National Research
Foundation, Prime Minister's Office, Singapore).

\appendix
\subsection{Relation between TOPP-RA and TOPP-NI}
\label{sec:relat-with-numer}

TOPP-RA and TOPP-NI are subtly related: they in fact compute the same
velocity profiles, but in different orders.  To facilitate this
discussion, let us first recall some terminologies from the classical
TOPP literature (refer to~\cite{pham2014general} for more details)
\begin{itemize}
\item Maximum Velocity Curve ($\MVC$): a mapping from a path position
  to the highest dynamically feasible velocity;
\item Integrate forward (or backward) following $\alpha$ (or $\beta$):
  for each tuple $s, \dot s$, $\alpha(s, \dot s)$ is the smallest
  control and $\beta(s, \dot s)$ is the greatest one; we integrate
  forward and backward by following the respective vector field
  ($\alpha$ or $\beta$);
\item $\alpha \rightarrow \beta$ switch point: there are three kinds
  of switch points: tangent, singular, discontinuous;
\item $\dot {s}_{\rm{beg}}, \dot s_{\rmend}$: starting and ending
  velocity at path positions $0$ and $ s_{\rmend}$ respectively.
\end{itemize}
Note that $\alpha, \beta$ functions recalled above are different from the
functions defined in Definition~\ref{def:alpha-beta}. The formers
maximize over the set of feasible states while the laters maximize
over the set of feasible \emph{and controllable states}.

TOPP-NI proceeds as follows
\begin{enumerate}
\item determine the  $\alpha\rightarrow\beta$ switch points;
\item from each $\alpha\rightarrow\beta$ switch point, integrate
  forward following $\beta$ and backward following $\alpha$ to obtain
  the Limiting Curves ($\LC$s);
\item take the lowest value of the $\LC$s at each position to form
  the Concatenated Limiting Curve ($\CLC$);
\item from $(0, \dot s_{\rm{beg}})$ integrate forward following
  $\beta$; from $(s_\rmend, \dot s_\rmend)$ integrate backward
  following $\alpha$ until their intersections with the $\CLC$; then
  return the combined $\beta-\CLC-\alpha$ profile.
\end{enumerate}

We now rearrange the above steps so as to compare with the two passes
of TOPP-RA, see Fig.~\ref{fig:CLC}
\begin{enumerate}
\item [] \emph{Backward pass}
\item [$1$)] determine the  $\alpha\rightarrow\beta$ switch points;
\item [${2a}$)] from each $\alpha\rightarrow\beta$ switch point, integrate
  backward following $\alpha$ to obtain the Backward Limiting Curves ($\BLC$s);
\item [${2b}$)] from the point $(s_{\rmend}, \dot s_{\rmend})$
  integrate backward following $\alpha$ to obtain the last BLC;
\item [$3$)] take the lowest value of the $\BLC$'s \emph{and the $\MVC$} at each
  position to form the upper boundary of the controllable sets.
\item [] \emph{Forward pass}
\item [$4a$)]set the current point to the point $(s_{\rm{beg}}, \dot s_{\rm{beg}})$;
\item [$4b$)] repeat until the current point is the point
  $(s_{\rmend}, \dot s_{\rmend})$, from the current point integrate
  forward following $\beta$ until hitting a BLC, set the corresponding
  switch point as the new current point.
\end{enumerate}

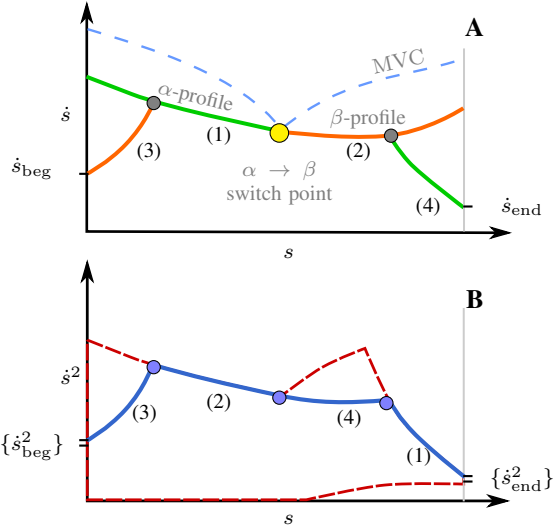
\begin{figure}[htb]
\centering
  \begin{tikzpicture}
    \node[anchor=south west,inner sep=0] (image) at (0,0)
    {\includegraphics[]{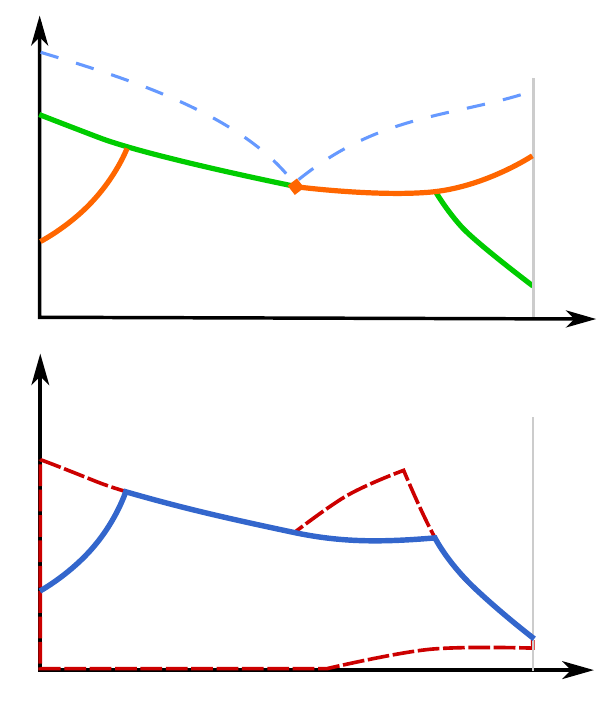}};
   \begin{scope}[x={(image.south east)},y={(image.north west)}]
     \node [axes label,rotate=-14, gray] at (0.3, 0.8) {$\alpha$-profile} ;
     \node [axes label,rotate=0, gray] at (0.67, 0.76) {$\beta$-profile} ;
    \node [rotate=13,axes label, anchor=east, text width=0, gray] at
    (0.7, 0.86) {$\MVC$};
    \node [axes label, anchor=west, text width=2cm] at
    (0.82, 0.60) {$\dot s_{\rmend}$};
    \draw [thick] (0.07, 0.66)  -- +(-0.02, 0.);
    \draw [thick] (0.875, 0.6)  -- +(0.02, 0.);
    \draw [thick] (0.07, 0.171)  -- +(-0.02, 0.)
               ++ (0.0, -0.01) -- + (-0.02, 0);
    \draw [thick] (0.875, 0.105) -- +(0.02, 0.)
               ++ (0.0, -0.01) -- + (0.02, 0);
    \node [axes label, anchor=west, text width=2cm] at
    (0.82,0.10) {$\{\dot s_{\rmend}^{2}\}$};
    \node [axes label, text width=2cm, anchor=east] at
    (0.13,0.67) {$\dot s_{\rm{beg}}$};
    \node [axes label, text width=2cm, anchor=east] at
    (0.13,0.16) {$\{\dot s_{\rm{beg}}^{2}\}$};
    \node [axes label, anchor=east, text width=0] at
    (0.03,0.77) {$\dot s$};
    \node [axes label, anchor=east, text width=0] at
    (0.03,0.29) {$\dot s^{2}$};
    \node [axes label, anchor=north] at (0.5, 0.05) {$s$};
    \node [axes label, anchor=north] at (0.5, 0.54) {$s$};
    \filldraw[fill=yellow] (0.48, 0.735) circle (3.5pt);
    \filldraw[fill=gray] (0.21, 0.79) circle (2.5pt);
    \filldraw[fill=gray] (0.72, 0.73) circle (2.5pt);
    \node[axes label, anchor=north, gray] (a) [text width=2cm] at
    (.48,0.7) {$\alpha\rightarrow\beta$ switch point};
    \node at (0.9, 0.93) {$\textbf{A}$};
    \node at (0.9, 0.43) {$\textbf{B}$};
    \node[axes label] at (0.2,  0.7) {(3)};
    \node[axes label] at (0.35, 0.73) {(1)};
    \node[axes label] at (0.65, 0.7) {(2)};
    \node[axes label] at (0.8,  0.6) {(4)};

    \filldraw[fill=blue!50] (0.21,   0.305) circle (2.5pt);
    \filldraw[fill=blue!50] (0.48, 0.249) circle  (2.5pt);
    \filldraw[fill=blue!50] (0.71,   0.239) circle (2.5pt);
    \node[axes label] at (0.19,  0.22) {(3)};
    \node[axes label] at (0.35, 0.24) {(2)};
    \node[axes label] at (0.63, 0.21) {(4)};
    \node[axes label] at (0.78,  0.14) {(1)};

    \end{scope}
  \end{tikzpicture}
  \caption{\label{fig:CLC} TOPP-NI \textbf{(A)} and TOPP-RA
    \textbf{(B)} compute the time-optimal path parameterization by
    creating similar profiles in different ordering. (1,2,3,4) are the
    orders in which profiles are computed. }
\end{figure}

The key idea in this rearrangement is to not compute $\beta$ profiles
immediately for each switch point, but delay until needed. The
resulting algorithm is almost identical to TOPP-RA except for the
following points
\begin{itemize}
\item since TOPP-RA does not require to explicitly compute the switch
  points (they are implicitly identified by computing the controllable
  sets) the algorithm avoids one of the major implementation
  difficulties of TOPP-NI;
\item TOPP-RA requires additional post-processing to remove the
  jitters. See the last paragraph of Section~\ref{sec:obta-joint-torq}
  for more details.
\end{itemize}


\subsection{Proof of optimality (no zero-inertia point)}
\label{sec:proof-no}

The optimality of TOPP-RA in this case relies on the properties of the
maximal transition functions.
\begin{definition}
  \label{def:alpha-beta}
  At a given stage $i$, the minimal and maximal controls at state $x$
  are defined by\,\footnote{These definitions differ from the
    common definitions of maximal and minimal controls\@. See Appendix A for more details.}
  \begin{equation*}
    \begin{aligned}
      \alpha_i(x) := \min\{u \mid \bar{\vec a}_i u + \bar {\vec b}_i x + \bar{\vec c}_i \leq 0 \}, \\
      \beta_i(x) := \max\{u \mid \bar{\vec a}_i u + \bar {\vec b}_i x + \bar{\vec c}_i \leq 0 \}.
    \end{aligned}
  \end{equation*}
  The \emph{minimal and maximal transition functions} are defined by
  \begin{equation*}
    T_i^\alpha(x) := x + 2 \Delta \alpha_i(x),\quad T_i^\beta(x) := x
    + 2 \Delta \beta_i(x). \qedhere
  \end{equation*}
\end{definition}

The key observation is: if the maximal transition function is
\emph{non-decreasing}, then the greedy strategy of TOPP-RA is optimal.
This is made precise by the following lemma.

\begin{lem}
\label{lem:main}
Assume that, for all $i$, the maximal transition function is
non-decreasing, i.e.
\[
  \forall x, x'' \in \mathcal{K}_i, \quad x \geq x' \implies
  T_i^\beta(x) \geq T^\beta_i(x').  
\]
Then TOPP-RA produces the optimal parameterization.
\end{lem}

\begin{proof}
  Consider an arbitrary admissible parameterization
  $\dot s_0^2 = x_0, u_0, \dots, u_{N-1}, x_N = \dot s_N^2$. We show
  by induction that, for all $i=0,\dots,N$, $x^*_{i} \geq x_i$, where
  the sequence $(x_{i}^*)$ denotes the parameterization returned by TOPP-RA
  (Algorithm~1). 

  \textrm{Initialization}: We have $x_0= \dot s_0^2 = x_0^*$, so the
  assertion is true at $i=0$.

  \textrm{Induction}: Steps 8 and 9 of Algorithm~1 can in fact be
  rewritten as follows
  \begin{equation*}
    x_{i+1}^* := \min \{ T_i^\beta(x_i^*), \max ({\cal K}_{i+1})\}.
  \end{equation*}
  By the induction hypothesis, one has $x_i^*\geq x_i$. Since
  $x_i^*, x_i\in \calK_i$, one has
  \[
  T_i^\beta (x_i^*) \geq T_i^\beta (x_i)  \geq x_{i+1}.
  \]
  Thus,
  \[
  \min \{ T_i^\beta(x_i^*), \max ({\cal K}_{i+1})\} \geq \min \{ x_{i+1},
  \max ({\cal K}_{i+1})\}, \ \textrm{i.e.}
  \]
  \[
  x_{i+1}^* \geq x_{i+1}.
  \]
  We have shown that at every stage the parameterization
  $x_0^*, \dots, x_N^*$ has higher velocity than that of any
  admissible parameterization. Hence it is optimal.
\end{proof}

Unfortunately, the maximal transition function is not always
non-decreasing, as made clear by the following lemma.

\begin{lem}
  \label{lem:1-basic}
  Consider a stage $i$, there exists $x^\beta_i$ such that, for all
  $x, x'\in \mathcal{K}_{i}$ 
  \begin{equation*}
    \begin{aligned}
    x \leq x' \leq x^\beta_i &\implies T_i^\beta(x) \leq T_i^\beta(x'), \\
    x^\beta_i \leq x \leq x' &\implies T_i^\beta(x) \geq T_i^\beta(x').
    \end{aligned}
  \end{equation*}
  In other words, $T^\beta_i$ is non-decreasing below $x^\beta_i$ and
  is non-increasing above $x^\beta_i$.

  Similarly, there exists $x^\alpha_i$ such that for  all $x, x'\in
  \mathcal{K}_{i}$ 
  \begin{equation*}
    \begin{aligned}
    x^\alpha_i \leq x \leq x' &\implies T_i^\alpha(x) \leq T_i^\alpha(x'), \\
    x \leq x' \leq x^\alpha_i &\implies T_i^\alpha(x) \geq T_i^\alpha(x').
    \end{aligned}
  \end{equation*}
\end{lem}
\begin{figure}[htp]
  \centering
  \begin{tikzpicture}
    \node[anchor=south west,inner sep=0] (image) at (0,0)
    {\includegraphics[]{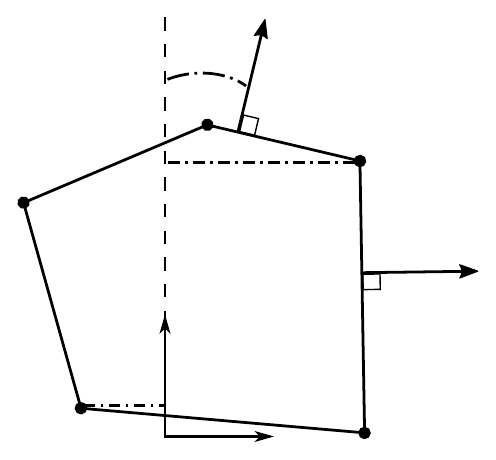}};
    \begin{scope}[x={(image.south east)},y={(image.north west)}]
      \node [anchor=west, text width=0] at (0.4, 0.9) {$\gamma$};
      \node [anchor=west, text width=0] at (0.53, 0.8) {$\vec v_{1}$};
      \node [anchor=west, text width=0] at (0.81, 0.45) {$\vec v_{2}$};
      \node [anchor=west, text width=0] at (0.33, 0.58) {$x^{\beta}_{i}$};
      \node [anchor=west, text width=0] at (0.22, 0.18) {$x^{\alpha}_{i}$};
      \node [axes label, anchor=west, text width=0] at (0.5, 0.0) {$u$};
      \node [axes label, anchor=west, text width=0] at (0.35, 0.3) {$x$};
    \end{scope}
  \end{tikzpicture}
  \caption{\label{fig:lemma1} At any stage, the polygon of
    controllable states and controls
    $\Omega_{i}\cap (\mathbb{R}\times \mathcal{K}_{i})$ contains
    $x^{\beta}_{i}$: the highest state under which the transition
    function $T^{\beta}_{i}$ is non-decreasing and $x^{\alpha}_{i}$:
    the lowest state above which the transition function $T^{\alpha}_{i}$
    is non-decreasing. }
\end{figure}
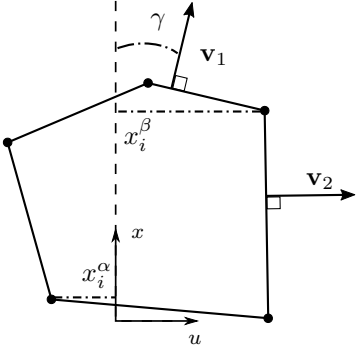

\begin{proof}
  Consider a state $x$. In the $(u,x)$ plane, draw a horizontal line
  at height $x$. This line intersects the polygon
  $\Omega_{i}\cap (\mathbb{R}\times \mathcal{K}_{i})$ at the minimal
  and maximal controls. See Fig~\ref{fig:lemma1}.

   
  Consider now the polygon $\Omega_{i}\cap (\mathbb{R}\times
  \mathcal{K}_{i})$.  Suppose that we enumerate the edges
  counter-clockwise (ccw), then the normals of the enumerated edges
  also rotate ccw. For example in Fig.~\ref{fig:lemma1}, the normal
  $v_1$ of edge 1 can be obtained by rotating ccw the normal $v_2$ of
  edge 2.

  Let $\gamma$ denote the angle between the vertical axis and the
  normal vector of the active constraint $k$ at $(x, \beta(x))$. One
  has $\cot{\gamma} = \bar{\vec b}_i[k] / \bar{\vec a}_i[k]$.  

  As $x$ increases, $\gamma$ decreases in the interval $(\pi, 0)$. Let
  $x^\beta_i$ be the lowest $x$ such that, for all $x > x^\beta_i$,
  $\gamma < \cot^{-1}{(1 / (2 \Delta))}$ ($x^\beta_i:=\max{\cal K}_i$ if
  there is no such $x$). 

  Consider now a $x > x^\beta_i$, one has, by construction
  \begin{equation}
    \label{eq:mono-cond}
    \frac{\bar{\vec b}_i[k]}{\bar{\vec a}_i[k]} >
    \frac{1}{2\Delta},
  \end{equation}
  where $k$ is the active constraint at $(x,\beta_i(x))$.
  The maximal transition function can be written as
  \begin{equation}
    \label{eq:2}
    \begin{aligned}
    T_i^\beta(x) &= x + 2 \Delta \beta_i(x)\\
    &= x + 2 \Delta \frac{- \bar{\vec c}_i[k] - \bar{\vec b}_i[k] x}{\bar{\vec a}_i[k]}\\
    &= x\left(1 - 2 \Delta \frac{\bar{\vec b}_i[k]}{\bar{\vec a}_i[k]}\right)
    - 2 \Delta \frac{\bar{\vec c}_i[k]}{\bar{\vec a}_i[k]}.
    \end{aligned}
  \end{equation}
  Since the coefficient of $x$ is negative, $T_i^\beta(x)$ is
  non-increasing. Similarly, for $x \leq x^\beta_i$, $T_i^\beta(x)$
  is non-decreasing. 
\end{proof}




We are now ready to prove Theorem~\ref{theo:1}.

\begin{proof}[Proof of Theorem~\ref{theo:1}]
  As there is no zero-inertia point, by uniform continuity, the
  $\bar{\vec a}(s)[k]$ are bounded away from 0. We can thus chose a
 step size $\Delta_\mathrm{thr}$ such that
  \begin{equation*}
    \frac{1}{2\Delta_\mathrm{thr}} > \max_{s, k} \left\{
        \frac{\bar{\vec b}(s)[k]}{\bar{\vec a}(s)[k]} \mid \bar{\vec
          a}(s)[k] >0   \right\}.
  \end{equation*}
  For any step size $\Delta < \Delta_\mathrm{thr}$, there is by
  construction no constraint that can have an angle
  $\gamma < \cot^{-1}{(1 / (2 \Delta))}$. Thus, for all stages $i$,
  $x^\beta_i=\max{\cal K}_i$, or in other words, $T_i^\beta(x)$ is
  non-decreasing in the whole set $\calK_i$. By Lemma~\ref{lem:main},
  TOPP-RA returns the optimal parameterization.
\end{proof}

\subsection{Proof of asymptotic optimality (with zero-inertia point)}
\label{sec:proof-with}

In the presence of a zero-inertia point, one cannot bound the
$\bar{\vec a}(s)[k]$ away from zero. Therefore, for any step size
$\Delta$, there is an interval around the zero-inertia point where the
maximal transition function is \emph{not} monotonic over the whole
controllable set $\calK_i$. Our proof strategy is to show that the
sub-optimality gap caused by that ``perturbation'' interval decreases
to $0$ with $\Delta$.

We first identify the ``perturbation'' interval. For simplicity,
assume that the zero-inertia point $s^\bullet$ is exactly at the
$i^\bullet$ grid point.

\begin{lem}
  \label{lem:perturbation}
  There exists an integer $l$ such that, for small enough $\Delta$,
  the maximal transition function is non-decreasing at all stages
  except in $[i^\bullet+1,\dots,i^\bullet+l]$.
\end{lem}

\begin{proof}
  Consider the Taylor expansion around $s^\bullet$ of the
  constraint that triggers the zero-inertia point
  \begin{equation*}
    \begin{aligned}
      \bar{\vec a}(s)[k] &= A'(s - s^\bullet) + o(s - s^\bullet), \\
      \bar{\vec b}(s)[k] &= B + B'(s - s^\bullet) + o(s - s^\bullet).
    \end{aligned}
  \end{equation*}

  Without loss of generality, suppose $A' > 0$.
  Eq.~\eqref{eq:mono-cond} can be written for stage
  $i^\bullet+r$ as follows
  \begin{equation*}
    2\Delta(B + B'r\Delta) > A' r \Delta + o(r\Delta).
  \end{equation*}
  Thus, in the limit $\Delta\to 0$, for
  $r> l := \mathrm{ceil}(2 B / A')$, Eq.~\eqref{eq:mono-cond}
  will not be fulfilled by constraint $k$. Using the construction of
  $\Delta_\mathrm{thr}$ in the proof of Theorem~\ref{theo:1}, one can
  next rule out all the other constraints at all stages.
\end{proof}

We now construct the ``perturbation strip'' by defining some
boundaries. See Fig.~\ref{fig:illus-box} for an illustration.

\begin{definition}
  Define states $(\kappa_i)_{i\in[i^\bullet+1,i^\bullet+l+1]}$ by
  \begin{equation*}
    \begin{aligned}
      \kappa_{i^\bullet+1}&:= x^\beta_{i^\bullet+1}, \\
      \kappa_i &:= \min ({T_{i-1}}^{\beta}(\kappa_{i-1}), x^\beta_i), \; i=i^\bullet+2, \dots, i^\bullet+l+1.
    \end{aligned}
  \end{equation*}
  Next, define $(\lambda_i)_{i\in[i^\bullet+1,i^\bullet+l+1]}$ by
  \begin{equation*}
    \begin{aligned}
      \lambda_{i^\bullet + 1} &= \kappa_{i^\bullet+1}, \\
      \lambda_i &= \min (T_{i-1}^{\alpha}(\kappa_{i-1}), x^\beta_i), \; i=i^\bullet+2, \dots, i^\bullet+l+1.
    \end{aligned}
  \end{equation*}

  Finally, define $(\mu_i)_{i\in[i^\bullet+1,i^\bullet+l+1]}$ as the
  highest profile that can be obtained by repeated applications of
  $T^\beta$ and that remains below the $(\lambda_i)$.
\end{definition}

\begin{figure}[htp]
  \centering
  \begin{tikzpicture}
    \node[anchor=south west,inner sep=0] (image) at (0,0)
    {\includegraphics[]{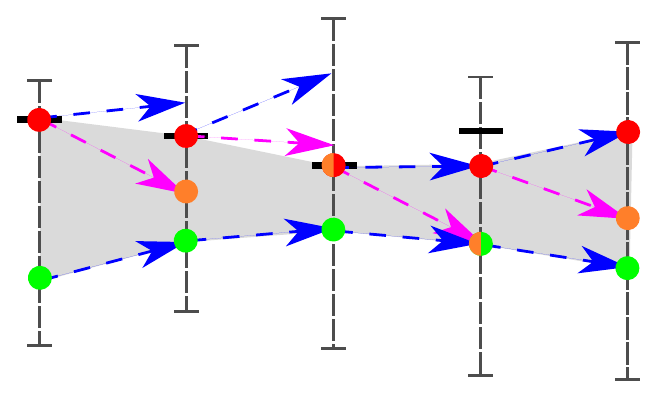}};
    \begin{scope}[x={(image.south east)},y={(image.north west)}]
      \node [axes label, anchor=north, text centered] at (0.05, 0.0) {$i^\bullet+1$};
      \node [axes label, anchor=north, text centered] at (0.28, 0.0) {$i^\bullet+2$};
      \node [axes label, anchor=north, text centered] at (0.5, 0.0) {$i^\bullet+3$};
      \node [axes label, anchor=north, text centered] at (0.72, 0.0) {$i^\bullet+4$};
      \node [axes label, anchor=north, text centered] at (0.95, 0.0) {$i^\bullet+5$};
      \node [axes label, anchor=west, text centered] at (0.8, 0.67) {$\kappa_{i^\bullet+5}$};
      \node [axes label, anchor=west, text centered] at (0.8, 0.45) {$\lambda_{i^\bullet+5}$};
      \node [axes label, anchor=west, text centered] at (0.8, 0.30) {$\mu_{i^\bullet+5}$};
      \node [axes label, anchor=west, text centered] at (0.57, 0.72) {$x^{\beta}_{i^\bullet+4}$};
    \end{scope}
  \end{tikzpicture}
  \caption{\label{fig:illus-box} The ``perturbation strip'' contains
    three vertical boundaries: $(\kappa_{i})$ [red dots],
    $(\lambda_{i})$ [orange dots] and $(\mu_{i})$ [green dots]. The
    states $(x^{\beta}_{i})$ [thick horizontal black lines] and the
    controllable sets $(\mathcal{K}_{i})$ [vertical intervals] are
    both shown.}
\end{figure}
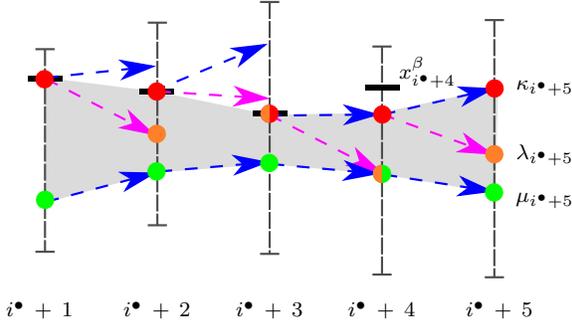

The $(\kappa_i)$ and $(\mu_i)$ form respectively the upper and the
lower boundaries of the ``perturbation strip''. Before going further,
let us establish some estimates on the size of the strip.

\begin{lem}
  \label{lem:boxbound}
  There exist constants $C_\kappa$ and $C_\mu$ such that, for all
  $i\in[i^\bullet+1,i^\bullet+l+1]$,
  \begin{equation}
    \label{eq:kappa}
    \max\calK_{i^\bullet+1} - \kappa_i \leq l C_\kappa
    \Delta,
  \end{equation}
  \begin{equation}
    \label{eq:mu}
    \max\calK_{i^\bullet+1} - \mu_i \leq l C_\mu \Delta.
  \end{equation}  
\end{lem}

\begin{proof}
  Let $C$ be the upper-bound of the absolute values of all admissible
  controls $\alpha,\beta$ over whole segment.  One has
  \[
  \calK_{i^\bullet+1} - \kappa_{i^\bullet+1} = \calK_{i^\bullet+1} -
  x^\beta_{i^\bullet+1} \leq
  \]
  \[(\beta_i(x^\beta_{i^\bullet+1}) -
  \alpha_i(x^\beta_{i^\bullet+1}))\tan(\gamma) \leq 2C\Delta.
  \]
  Next, by definition of $\kappa$, one can see that the difference
  between two consecutive $\kappa_i,\kappa_{i+1}$ is bounded by
  $2C\Delta$. This shows Eq.~\eqref{eq:kappa}.

  Since $(\mu_{i})$ is the highest profile below $(\lambda_{i})$, there exists
  one index $p$ such that $\mu_p=\lambda_p$. Thus,
  $\kappa_p-\mu_p = \kappa_p-\lambda_p \leq 2C\Delta$, where the last
  inequality comes from the definition of $\lambda$. Remark finally
  that the difference between two consecutive $\mu_i,\mu_{i+1}$ is also
  bounded by $2C\Delta$. This shows Eq.~\eqref{eq:mu}.
\end{proof}

We now establish the fundamental properties of the ``perturbation strip''.

\begin{lem}[and definition]
  \label{lem:boxprop}
  Let $J_i^*(x)$ denote TOPP-RA's  cost-to-go: the cost of the profile
  produced by TOPP-RA starting from $x$ at the $i$-stage, and
  $J_{i}^{\dagger}(x)$ the optimal cost-to-go.
  \begin{enumerate}
  \item[(a)] In the interval $[\min \mathcal{K}_{i}, \mu_i]$,
    $J_i^*(x)$ equals $J_i^\dagger(x)$ and
    is non-increasing;
  \item[(b)] For all $i \in [i^\bullet + 1, \dots, i^\bullet + l$], 
  \begin{equation}
    \label{eq:9}
    x \in [\mu_i , \kappa_i] \implies T_i^\dagger(x), T_i^\beta(x) \in [ \mu_{i+1} , \kappa_{i+1}],
  \end{equation}
  where $T^\dagger_i(x)$ is the optimal transition.
  \end{enumerate}
\end{lem}

\begin{proof}
  (a) We use backward induction from $i^\bullet + l$ to
  $i^\bullet + 1$. 

  Initialization: One has
  $x \leq \mu_{i^\bullet + l} \leq \lambda_{i^\bullet + l} \leq
  x^\beta_{i^\bullet + l}$.
  It follows that $T_{i^\bullet + l}^{\beta}(x)$ is non-decreasing
  over the interval
  $[\min \mathcal{K}_{i^\bullet+l}, \mu_{i^\bullet+l}]$.

  As there is no constraint verifying Eq.~\eqref{eq:mono-cond} at
  stages $i=i^\bullet + l + 1, \dots, N$, the cost-to-go
  $J^*_{i^\bullet + l + 1}(x)$ is non-increasing and equals
  the optimal cost-to-go $J^\dagger_{i^\bullet + l + 1}(x)$
  by Theorem~\ref{theo:1}. Choosing the greedy control at the
  $i^\bullet+l$-stage is therefore optimal.  Next, note that
  \begin{equation}
    \label{eq:4b}
    J^{*}_{i^\bullet + l}(x) = \frac{\Delta}{\sqrt{x}}
    + J^{*}_{i^\bullet + l + 1}(T_{i^\bullet + l}^\beta(x)),
  \end{equation}
  since $J^{*}_{i^\bullet + l+1}(x)$ and
  $T_{i^\bullet + l}^{\beta}(x)$ are non-increasing and non-decreasing
  respectively over
  $[\min \mathcal{K}_{i^\bullet+l+1}, \mu_{i^\bullet+l+1}]$ and
  $[\min \mathcal{K}_{i^\bullet+l}, \mu_{i^\bullet+l}]$, it follows
  that $J^{*}_{i^\bullet + l}(x)$ is non-increasing over the
  interval $[\min \mathcal{K}_{i^\bullet + l}, \mu_{i^\bullet+l}]$.

  Induction: Suppose the hypothesis is true for
  $i+1 \in \{i^\bullet+2,\dots,i^\bullet+l\}$. Since
  $x \leq \mu_i \leq \lambda_i \leq x^\beta_i$, one has that
  $T_i^\beta(x) \leq \mu_{i+1}$ and that $T_i^{\beta}(x)$ is
  non-decreasing over the interval $[\min\mathcal{K}_i, \mu_i]$. Note
  that
  \begin{equation}
    \label{eq:5}
    J^*_{i}(x) = \frac{\Delta}{\sqrt{x}}
    + J^*_{i + 1}(T_{i}^\beta(x)).
  \end{equation}
  By the induction hypothesis, $J^{*}_{i + 1}(T_{i}^\beta(x))$
  is non-increasing, it then follows that $J^{*}_{i}(x) $ is
  non-decreasing and that $\beta_{i}(x)$ is the optimal control and thus
  $J_{i}^{*}(x) = J_{i}^\dagger(x)$.

  (b) The part that $T_i^\dagger(x), T_i^\beta(x) \leq \kappa_{i+1}$
  is clear from the definition of $\kappa$. We first show
  $\mu_{i+1} \leq T_i^\beta(x)$.

  Suppose first $x \leq x^\beta_i$. Then $T_i^\beta$ is
  non-decreasing in $[\mu_i,x]$, which implies $T_i^\beta(x) \geq
  T_i^\beta(\mu_i) = \mu_{i+1}$. 

  Suppose now $x \geq x^\beta_i$. One can choose a step size
  $\Delta$ such that $x^\alpha_i < x^\beta_i$, which implies
  $x > x^\alpha_i$. One then has
  $T_i^\beta(x) \geq T_i^\alpha(x)\geq T_i^\alpha(x^\beta_i) \geq
  \lambda_{i+1} \geq \mu_{i+1}$.

  Finally, to show that $\mu_{i+1} \leq T_i^\dagger(x)$, we reason by
  contradiction. Suppose  $T_i^\dagger(x) < \mu_{i+1}$. By (a),
  $J^\dagger_{i+1}$ is non-increasing below $\mu_{i+1}$,
  thus $J^\dagger_{i+1}(T_i^\dagger(x)) >
  J^\dagger_{i+1}(\mu_{i+1})$ (*). On the other hand, since
  $T_i^\dagger(x) < \mu_{i+1} \leq T_i^\beta(x)$, there exists an
  admissible control that steers $x$ towards $\mu_{i+1}$. Since
  $T_i^\dagger$ is the true optimal transition from $x$,
  $J^\dagger_{i+1}(\mu_{i+1}) \geq
  J^\dagger_{i+1}(T_i^\dagger(x))$. This contradicts (*).
\end{proof}

We are now ready to prove Theorem~\ref{theo:2}.

\begin{proof}[Proof of Theorem~\ref{theo:2}]

  Recall that $(\dot s_0^2=x^*_0, \dots, x^*_N)$ is the profile
  returned by TOPP-RA and
  $(\dot s_0^2 = x^\dagger_0, \dots, x^\dagger_N)$ is the true optimal
  profile. By definition of the time-optimal cost functions, we can
  expand the initial costs $J_0^*(\dot{s}_0^2)$ and
  $J_0^\dagger(\dot{s}_0^2)$ into three terms as follows
  \begin{equation}
    \label{eq:vstar}
    J_0^*(\dot{s}_0^2) =
    \sum_{i=0}^{i^\bullet}\frac{\Delta}{\sqrt{x_i^*}} + 
    \sum_{i=i^\bullet+1}^{i^\bullet+l}\frac{\Delta}{\sqrt{x_i^*}}
    + J_{i^\bullet+l+1}^*(x_{i^\bullet + l+1}^*),
  \end{equation}
  and 
  \begin{equation}
    \label{eq:vdagger}
    J_0^\dagger(\dot{s}_0^2) =
    \sum_{i=0}^{i^\bullet}\frac{\Delta}{\sqrt{x_i^\dagger}} + 
    \sum_{i=i^\bullet+1}^{i^\bullet+l}\frac{\Delta}{\sqrt{x_i^\dagger}}
    + J_{i^\bullet+l+1}^\dagger(x_{i^\bullet + l+1}^\dagger),
  \end{equation}

  (a) Applying Theorem~\ref{theo:1}, for small enough $\Delta$, one can
  show that
  \begin{equation}
    \label{eq:8}
    \forall i \in [0, \dots, i^\bullet+1],\ x^*_i \geq x^\dagger_i.
  \end{equation}
  Thus, the first term of $J_0^*(\dot{s}_0^2)$ is smaller
  than the first term of $J_0^\dagger(\dot{s}_0^2)$.

  (b) Suppose
  $x^{*}_{i^{\bullet}+1}, x^{\dagger}_{i^{\bullet}+1} \in [\mu_{i^\bullet+1},
  \kappa_{i^\bullet+1}]$. From Lemma~\ref{lem:boxprop}(b), one has for all
  $i\in[i^\bullet + 1, \dots, i^\bullet + l]$,
  $x^*_i,x^\dagger_i\in [\mu_i , \kappa_i]$. Thus, using the estimates
  of Lemma~\ref{lem:boxbound}, the second terms can be bounded as follows 
  \[
  \sum_{i=i^\bullet+1}^{i^\bullet+l}\frac{\Delta}{\sqrt{x_i^*}}
  \leq \frac{l\Delta}{\sqrt{\max\mathcal{K}_{i^\bullet+1} -
      C_\mu\Delta l}},\ \mathrm{and}
  \]
  \[
  \sum_{i=i^\bullet+1}^{i^\bullet+l}\frac{\Delta}{\sqrt{x_i^\dagger}}
  \geq \frac{l\Delta}{\sqrt{\max\mathcal{K}_{i^\bullet+1} +
      C_\kappa\Delta l}}.
  \]
  Thus
  \[
 \sum_{i=i^\bullet+1}^{i^\bullet+l}\frac{\Delta}{\sqrt{x_i^*}} -
 \sum_{i=i^\bullet+1}^{i^\bullet+l}\frac{\Delta}{\sqrt{x_i^\dagger}}
 \leq
 \frac{(C_\mu+C_\kappa) \Delta^2l^2}{2\sqrt{\max \mathcal{K}_{i^\bullet+1} - C_\mu\Delta l}}.
  \]

  If
  $x^*_{i^{\bullet}+1}, x^\dagger_{i^{\bullet}+1} <
  \mu_{i^{\bullet}+1}$, by Lemma~\ref{lem:boxprop}(a) it is easy to
  see that $J_{0}^*(\dot s_{0}^{2}) = J_{0}^\dagger(\dot
  s_{0}^{2})$. 

  If $x^*_{i^{\bullet}+1} \geq \mu_{i^{\bullet}+1} >
  x^\dagger_{i^{\bullet}+1}$, then by Lemma~\ref{lem:boxprop},
  $x^*_{i^{\bullet}+l+1}\geq \mu_{i^{\bullet}+l+1} >
  x^\dagger_{i^{\bullet}+l+1}$, which implies next that
  $J_{0}^*(\dot s_{0}^{2}) < J_{0}^\dagger(\dot s_{0}^{2})$, which
  is impossible. 

  (c) Regarding the third terms, observe that, by applying
  Theorem~\ref{theo:1} over $[i^{\bullet}+l+1,\dots,N]$, one has
   $J^{*}_{i^{\bullet}+l+1}(x) = J^{\dagger}_{i^{\bullet}+l+1}(x)$ for
  all $x\in\mathcal{K}_{i^{\bullet}+l+1}$. Thus
  \[
  J^*_{i^\bullet+l+1}(x^*_{i^\bullet+l+1}) -
  J^\dagger_{i^\bullet+l+1}(x^\dagger_{i^\bullet+l+1}) = 
  \]
  \[
  J^\dagger_{i^\bullet+l+1}(x^*_{i^\bullet+l+1}) -
  J^\dagger_{i^\bullet+l+1}(x^\dagger_{i^\bullet+l+1}) \leq
  \]
  \[
  C_{J^\dagger} |x^*_{i^\bullet+l+1}-x^\dagger_{i^\bullet+l+1}| \leq
  C_{J^\dagger}(C_\mu+C_\kappa)\Delta l,
  \]
  where $C_{J^\dagger}$ is the Lipshitz constant of $J^\dagger$.
  
  Grouping together the three estimates (a), (b), (c) leads to the
  conclusion of the theorem.
\end{proof}

\subsection{Error analysis for different discretization schemes}
\label{sec:discr-scheme-error}

\subsubsection{First-order interpolation scheme}

In the main text the collocation scheme was presented to discretize
the constraints.  Before analyzing the errors, we introduce another
scheme: first-order interpolation.

In this scheme, at stage $i$ we require $(u_{i}, x_{i})$ and
$(u_{i}, x_{i} + 2 \Delta_{i} u_{i})$ to satisfy the constraints at $s=s_i$ and
$s=s_{i+1}$ respectively. That is, for $i=0,\dots,N-1$
\begin{equation}
  \label{eq:first-order-inerpo}
  \begin{aligned}
    \begin{bmatrix}
      \vec a(s_i)  \\
      \vec a(s_{i+1}) + 2 \Delta\vec b(s_{i+1})
   \end{bmatrix} u
   +
   \begin{bmatrix}
     \mathbf{b}(s_i)  \\
     \mathbf{b}(s_{i+1})  \\
   \end{bmatrix}
   x+
   \begin{bmatrix}
     \mathbf{c}(s_i)   \\
     \mathbf{c}(s_{i+1})   \\
   \end{bmatrix} \in
    \\
   \begin{bmatrix}
     \cC(s_i)\\
     \cC(s_{i+1})
   \end{bmatrix}
  \end{aligned}
\end{equation}
At $i=N$, one uses only the top half of the above equations. By
appropriately rearranging the terms, the above equations can finally
be rewritten as
\begin{equation}
  \label{eq:gen-form2}
    \mathbf{a}_i u + \mathbf{b}_i x+ \mathbf{c}_i \in \cC_i.
\end{equation}

\subsubsection{Error analysis}
\label{sec:error-analysis}
For simplicity, suppose Assumption~\ref{assum:1} holds. That is, there
exists
$\tilde{\vec a}(s)_{s\in [0, 1]}, \tilde{\vec b}(s)_{s\in [0, 1]},
\tilde{\vec c}(s)_{s\in [0, 1]}$ which define the set of admissible
control-state pairs. 

On the interval $[s_0, s_1]$, the parameterization is given by
\begin{equation*}
  x(s; u_0, x_0) = x_0 + 2 s u_0,
\end{equation*}
where $x_0$ is the state at $s_0$ and $u_0$ is the constant control
along the interval.
Additionally, note that $s_0 = 0, s_1=\Delta$.

The constraint satisfaction \emph{function} is defined by
\begin{equation}
 \label{eq:err-f}
\vec{\epsilon}(s):= u_{0}\tilde{\vec a}(s) + x(s)\tilde{\vec b}(s) +
\tilde{\vec c}(s).
\end{equation}
The \emph{greatest} constraint satisfaction error over
$[s_{0}, s_{1}]$ can be given as
\begin{equation*}
  \max\left\{\max_{k,s\in[s_{0}, s_{1}]}\vec \epsilon(s)[k], 0\right\}.
\end{equation*}
 
Different discretization schemes enforce different conditions on
$\vec \epsilon(s)$.  In particular, we have
\begin{itemize}
\item \emph{collocation scheme}: $\vec \epsilon(s_0) \leq 0$;

\item \emph{first-order interpolation scheme}: $\vec \epsilon(s_0)
  \leq 0$, $\vec \epsilon(s_{1}) \leq 0$. 
\end{itemize}

Using the classic result on Error of Polynomial
Interpolation~\cite[Theorem 2.1.4.1]{Stoer1982}, we obtain the
following estimation of $\vec \epsilon(s)$ for the \emph{collocation
  scheme}:
\begin{equation*}
  \vec \epsilon(s) = \vec \epsilon(s_0) + (s - s_0) \vec \epsilon'(\xi) = s \vec \epsilon'(\xi),
\end{equation*}
for $\xi \in [s_0, s]$. Suppose the derivatives of
$\tilde{\vec a}(s), \tilde{\vec b}(s), \tilde{\vec c}(s)$ are bounded
we have
\begin{equation*}
  \max_{s\in[s_{0}, s_{1}]}\vec \epsilon(s) = O(\Delta).
\end{equation*}
Thus the greatest constraint satisfaction error of the collocation
discretization scheme has order $O(\Delta)$.

Using the same theorem, we obtain the following estimation of
$\vec\epsilon(s)$ for the \emph{first-order interpolation scheme}:
\begin{equation*}
  \begin{aligned}
  \vec \epsilon(s) = &\vec \epsilon(s_0) + (s - s_0)\frac{\vec \epsilon(s_{1}) - \vec \epsilon(s_0)}{s_1-s_0}
   +\frac{(s-s_0)(s - s_{1}) \vec \epsilon''(\xi)}{2!} \\
  =& \frac{s(s - \Delta) \vec \epsilon''(\xi)}{2!},
  \end{aligned}
\end{equation*}
for some $\xi \in [s_0, s_1]$. Again, since the derivatives of
$\tilde{\vec a}(s), \tilde{\vec b}(s), \tilde{\vec c}(s)$ are assumed
to be bounded, we see that
\begin{equation*}
  \max_{s\in[s_{0}, s_{1}]}\vec \epsilon(s) = O(\Delta^2).
\end{equation*}
Thus the greatest constraint satisfaction error of the first-order
interpolation discretization scheme has order $O(\Delta^{2})$.

\bibliographystyle{IEEEtran}
\bibliography{library}

\end{document}